\documentclass[runningheads]{llncs}

\usepackage{graphicx}
\usepackage{amsmath} 
\usepackage{amssymb}
\usepackage{physics}
\usepackage[utf8]{inputenc}
\usepackage[english]{babel}
\usepackage{dsfont}
\usepackage{extarrows}

\DeclareMathOperator*{\argmin}{arg\,min}
\DeclareMathOperator{\Sup}{\textup{Sup}}

\title{Matrix-Valued LogSumExp Approximation for Colour Morphology}
\author{Marvin Kahra\inst{1}, Michael Breu\ss\inst{1}, Andreas Kleefeld\inst{2,3} and Martin Welk\inst{4}}
\authorrunning{M. Kahra et al.}
\institute{Institute for Mathematics, Brandenburg University of Technology Cottbus-Senftenberg, 03046 Cottbus, Germany \\
\email{\{marvin.kahra,breuss\}@b-tu.de} \and
Forschungszentrum Jülich GmbH, Jülich Supercomputing Centre, \\
Wilhelm-Johnen-Stra\ss e, 52425 Jülich, Germany \\
\email{a.kleefeld@fz-juelich.de} \and
University of Applied Sciences Aachen, Faculty of Medical Engineering and Technomathematics, Heinrich-Mu\ss{}mann-Str. 1, 52428 J\"ulich, Germany \and
UMIT TIROL – Private University for Health Sciences and Health
Technology, Eduard-Wallnöfer-Zentrum 1, 6060 Hall/Tyrol, Austria \\
\email{martin.welk@umit-tirol.at}
}

\begin{document}

\maketitle

\begin{abstract}

    Mathematical morphology is a part of image processing that uses a window that moves across the image to change certain pixels according to certain operations. The concepts of supremum and infimum play a crucial role here, but it proves challenging to define them generally for higher-dimensional data, such as colour representations. Numerous approaches have therefore been taken to solve this problem with certain compromises.

    In this paper we will analyse the construction of a new approach, which we have already presented experimentally in paper \cite{DGMM24}. This is based on a method by Burgeth and Kleefeld \cite{BurgethKleefeld}, who regard the colours as symmetric $2\times2$ matrices and compare them by means of the Loewner order in a bi-cone through different suprema. However, we will replace the supremum with the LogExp approximation for the maximum instead. This allows us to transfer the associativity of the dilation from the one-dimensional case to the higher-dimensional case. In addition, we will investigate the minimality property and specify a relaxation to ensure that our approach is continuously dependent on the input data.

\end{abstract}

\section{Introduction}

Mathematical morphology is a theory used to analyse spatial structures in images. Over the decades it has developed into a very successful field of image processing, see e.g. \cite{Najman-Talbot,Roerdink-2011,Serra-Soille} for an overview. Morphological operators basically consist of two main components. The first of these is the structuring element (SE), which is characterised by its shape, size and position. These in turn can be divided into two types of SEs, flat and non-flat cf. \cite{Haralick_1}. A flat SE basically defines a neighbourhood of the central pixel where appropriate morphological operations are performed, while a non-flat SE also contains a mask with finite values that are used as additive offsets. The SE is usually implemented as a window sliding over the image. The second main component is used to perform a comparison of values within an SE. The basic operations in mathematical morphology are dilation and erosion, where a pixel value is set to the maximum and minimum, respectively, of the discrete image function within the SE. Many morphological filtering procedures of practical interest, such as opening, closing or top hats, can be formulated by combining dilation and erosion operations. Since dilation and erosion are dual operations, it is often sufficient to restrict oneself to one of the two when constructing algorithms.

Let us also briefly extend this concept to colour morphology, as it is the underlying concept for our further considerations. As already mentioned, the most important operation in morphology is to perform a comparison of the tonal values or, in our case, the colour values within the SE over certain sets of pixels in an image domain. 
For the simpler application areas such as binary or grey value morphology, one can act directly on complete lattices in order to obtain a total order of the colour values, see \cite{Serra-Soille}. In the case of colour morphology, this is no longer the case, as there is no total order of the colour values. For this reason, corresponding semi-orders and different basic structures are used, cf. \cite{BAR76}. The first approach that could be used for this would be to regard each colour channel of an image as an independent image and to perform grey value morphology on each of them. This approach has the serious disadvantage that we lose the correlated information between the colour channels, which could be used to further improve the filtering results. The other approach, which is more popular, uses a vector space structure in which each colour is considered a vector in an underlying colour space. 
In order to compute a supremum or infimum, it is necessary to have an order for the vector space. However, there are a plethora of ordering approaches for colour morphology. For details of the most commonly used approaches, we refer the reader to the overview provided in \cite{overview_orderings}.
We will take the latter approach, but use symmetric matrices instead of vectors. Since there is also no total order for colour matrices either, we will order the elements by means of a semi-order, namely the Loewner order, see \cite{Loewner}. This means, though, that we need an additional function to select a minimum upper bound, namely the supremum function.

To calculate two of the basic operations of colour morphology, dilation and erosion, it is necessary to determine the supremum or infimum. Because of the duality between these two operations, it is common to consider only one of them. Here we will concentrate on dilation and the construction of the supremum. However, there are several approaches how to choose the supremum of a set of symmetric matrices, based on different norms. To give some examples, we mention here the nuclear norm, the Frobenius norm and the spectral norm. For a comparison of these norms we refer to the work by Welk, Kleefeld and Breu\ss\ \cite{WelkQuantile}. 

Here we want to consider another approach, namely the approximation of the supremum by a so-called LogSumExp approximation of Maslov \cite{Maslov}. This is an approximation which has already given promising results in the work of Kahra, Sridhar and Breu{\ss} \cite{KSB} for grey-scale images and for colour images in \cite{Srid} in conjunction with a fast Fourier transform. However, the latter only represents a one-dimensional channel-wise approach to colour morphology. Another connection worth mentioning is the work \cite{morph_op_mat_im} of Burgeth, Welk, Feddern and Weickert, where root and power functions were used for symmetric positive semidefinite matrices instead of logarithm and exponential function. However, our approach does not require a positive semidefinite matrix, but works with any colour matrix, and preserves the so-called transitivity of grey-scale morphology.

This paper will be the next step from \cite{DGMM24} for transferring the LogSumExp approach to colour morphology with tonal vectors/matrices. The goal is to present a clear characterisation of this approach for tonal value matrices to close the gap in the reasoning of \cite{DGMM24} and to extend it with regard to certain properties. In this way, we will end up with a dilation operator that, with a few minor compromises, combines many of the advantageous properties of the other multidimensional approaches while preserving the associativity of the dilation, which, as far as we know, is not the case with the other multidimensional approaches. In addition, we will present a relaxed formulation of the operator, which addresses one of the primary limitations of the operator.

\section{General Definitions}

To make this paper self-contained, we want to use this section to clarify some basic definitions and terminology, using our previous paper \cite{DGMM24} as a guide. This is divided into two subsections, one for the morphological terms and one for the terms related to the Loewner order. 

We will start with the morphological concepts of dilation/erosion for grey-scale images and then extend this to colour images. In particular, we will discuss how colours can be represented and which algebraic structures we will consider for this paper and describe the one we have chosen in more detail.

In the second subsection, we look at why we need such an order and why we have chosen it. We show what we mean by a minimiser of a convex set with respect to the Loewner order and which properties it must fulfil. This will give us a general explanation for a matrix supremum. Finally, we will give a brief overview of how different norms lead to different matrix suprema.

\subsection{Colour Morphology}

We begin with a two-dimensional, discrete image domain $\Omega \subseteq \mathbb{Z}^2$ and a single-channel grey-scale image, which is described by a function $f: \Omega \rightarrow [0,255]$. In the case of non-flat morphology, the \textbf{structuring element (SE)} can be represented as a function $b: \mathbb{Z}^2 \rightarrow \mathbb{R} \cup \{ - \infty\}$ with 
\begin{align}
    b(\boldsymbol x) := 
    \begin{cases}
        \beta(\boldsymbol x),  & \boldsymbol x \in B_0, \\
		- \infty, & \text{otherwise},
    \end{cases}
    \quad B_0 \subset \mathbb{Z}^2,
    \label{SE}
\end{align}
where $\beta$ is itself a small grey-scale image that should have the same scaling for the grey-scale values as the input image. It defines the height of the SE or in other words which pixel will be ``prioritised" by the filtering. $B_0$ is a set centred at the origin. However, the origin of the SE needs not always to be at its centre. It determines the shape and size of the SE, as it specifies which elements are to be compared with each other. Frequently used shapes for this are squares, discs, diamonds, hexagons and crosses, see \cite{Boroujerdi,MMSoille}. In the case of a flat filter (flat morphology), it is simply the special case $\beta(\boldsymbol x) = 0$. Two of the most elementary operations of mathematical morphology are \textbf{dilation} 
\begin{align}
    (f\oplus b)(\boldsymbol x) := \max_{\boldsymbol u\in \mathbb{Z}^2}  {\{f(\boldsymbol x- \boldsymbol u)\: +\: b(\boldsymbol u) \}}, \quad \boldsymbol{x} \in \Omega,
    \label{dilation}
\end{align}
and \textbf{erosion}
\begin{align}
    (f\ominus b)(\boldsymbol x) := \min_{\boldsymbol u\in \mathbb{Z}^2}  {\{f(\boldsymbol x + \boldsymbol u)\: -\: b(\boldsymbol u) \}}, \quad \boldsymbol{x} \in \Omega,
    \label{erosion}
\end{align}
see Figure \ref{fig:greyscale_operations} for an example. In particular, these two operations are dual in the following sense with respect to complementation. Let the range of the grey-scale values of $f$ be given by the interval $[f_{\min},f_{\max}] \subseteq [0,255]$, where $f_{\min}$ is the lower limit and $f_{\max}$ the upper limit of the grey-scale values of $f$. We define the \textbf{complementary image} $f^c$ as
\begin{align}
    f^c(\boldsymbol{x}) := f_{\max} - f(\boldsymbol{x}) + f_{\min}, \quad \boldsymbol{x} \in  \Omega.
    \label{comp_image}
\end{align}
Then one has
\begin{align}
    \left( f^c \oplus \Breve{b}\right)^c (\boldsymbol{x}) &= f_{\max} -  \max_{\boldsymbol u\in \mathbb{Z}^2} \left\{f_{\max} - f(\boldsymbol{x} - \boldsymbol{u}) + \Breve{b}(\boldsymbol{u}) + f_{\min}\right\} + f_{\min} \nonumber \\ \nonumber
    &= f_{\max} - f_{\max} - \max_{\boldsymbol u\in \mathbb{Z}^2} \left\{- f(\boldsymbol{x} - \boldsymbol{u}) + \Breve{b}(\boldsymbol{u})\right\} - f_{\min} + f_{\min} \\ \nonumber
    &=  \min_{\boldsymbol u\in \mathbb{Z}^2} \left\{f(\boldsymbol{x} - \boldsymbol{u}) - \Breve{b}(\boldsymbol{u})\right\} = \min_{\boldsymbol u\in \mathbb{Z}^2} \left\{f(\boldsymbol{x} + \boldsymbol{u}) - b(\boldsymbol{u})\right\} \\ 
    &= (f\ominus b)(\boldsymbol x), \quad \boldsymbol{x} \in \Omega,
    \label{duality_dil-ero}
\end{align}
where $\Breve{b}(\boldsymbol{x}) = b(-\boldsymbol{x})$. This shows the duality between dilation and erosion.

However, with these two operations, many other operations can be defined that are of great interest in practice, e.g. \textbf{opening} $f \circ b = (f \ominus b) \oplus b$ and \textbf{closing} $f \bullet b = (f \oplus b) \ominus b$, see Figure \ref{fig:greyscale_operations}. In general, an opening will result in the deletion of minor, protruding components of an object. For example, this can be observed by the first column on the left side of the entrance in Figure \ref{fig:greyscale_operations}. The closing operation will fill small holes or thin intruding parts of the object. This may entail the destruction of smaller dark areas, such as the shadows of the columns on the outermost left side or the thin flat window illustrated in Figure \ref{fig:greyscale_operations}. These filtering operations can be employed in conjunction with varying sizes of the SE to compute size distributions in binary images, see \cite{Najman-Talbot} for further information about granulometry and related applications.

\begin{figure}[t]
\centering
\minipage{0.18\linewidth}
    \includegraphics[width=\linewidth]{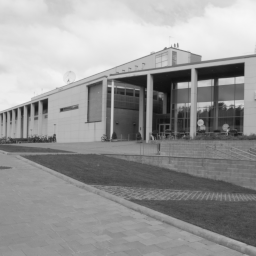}
\endminipage
\hspace{0.2cm}
\minipage{0.18\linewidth}
    \includegraphics[width=\linewidth]{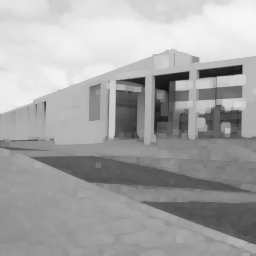}
\endminipage
\hspace{0.2cm}
\minipage{0.18\linewidth}
    \includegraphics[width=\linewidth]{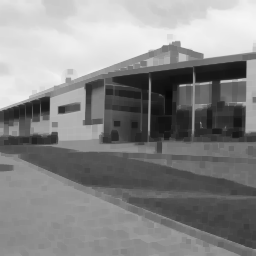}
\endminipage
\hspace{0.2cm}
\minipage{0.18\linewidth}
    \includegraphics[width=\linewidth]{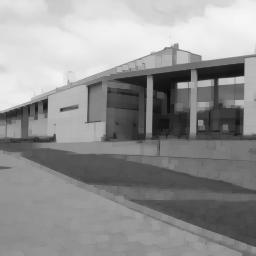}
\endminipage
\hspace{0.2cm}
\minipage{0.18\linewidth}
    \includegraphics[width=\linewidth]{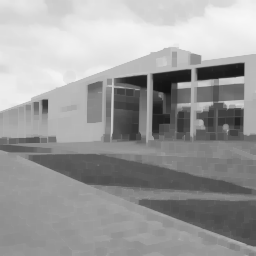}
\endminipage
\caption{\label{fig:greyscale_operations} Application of different morphological operators on a $256 \times 256$ grey-scale image with a $5 \times 5$ SE.
{\bf From left to right:} Downscaled original image from TAMPERE17 noise-free image database \cite{imageDB}, dilation, erosion, opening and closing.
}
\end{figure}

We turn now to our actual area of interest, namely colour morphology. This is similar to the grey-scale morphology already shown, with the difference that we no longer have just one channel, but three. There are many useful formats for expressing this, see \cite{SHA03}. A classic approach in this sense is the channel-by-channel processing of an image with the Red-Green-Blue (RGB) colour model, see e.g. \cite{Srid} for a recent example of channel-wise scheme implementation. However, instead of RGB vectors, we will use symmetric $2 \times 2$ matrices. 

For this we assume that the colour values are already normalised to the interval $[0,1]$ for the corresponding channels. First, we transfer this vector into the Hue-Chroma-Luminance (HCL) colour space by means of $M = \max\{R,G,B\}$, $m = \min\{R,G,B\}$, $C = M - m$, $L = \frac{1}{2} (M + m)$ and 
\begin{align*}
    H = 
    \begin{cases}
        \frac{G - B}{6C} \mod 1, & \textup{ if } M = R, \\
        \frac{B - R}{6C} + \frac{1}{3} \mod 1, & \textup{ if } M = G, \\
        \frac{R - G}{6C} + \frac{2}{3} \mod 1, & \textup{ if } M = B.
    \end{cases}
\end{align*}
Subsequently, the luminance $L$ is replaced with the modified luminance $\tilde{L} = 2L -1$ and the quantities $C$, $2\pi H$ and $\tilde{L}$ are regarded as radial, angular and axial coordinates of a cylindrical coordinate system, respectively. Since the transformation shown here maps each colour from the RGB colour space one-to-one to a colour in the HCL bi-cone, see Figure \ref{fig:bicone}, this represents a bijection onto the bi-cone, which in turn is interpreted with Cartesian coordinates by $x = C \cos(2\pi H)$, $y = C \sin(2\pi H)$ and $z = \tilde{L}$. Finally, we map these Cartesian coordinates onto a symmetric matrix in the following manner:
\begin{align*}
    \boldsymbol A := \frac{\sqrt{2}}{2} 
    \begin{pmatrix}
        z - y \hspace{0.5em}& x \\
        x \hspace{0.5em}& z + y
    \end{pmatrix},
\end{align*}
where the complete transformation process is a bijective mapping, see \cite{Loewner}.

\begin{figure}
    \centering
    \begin{minipage}{0.4\linewidth}
        \includegraphics[width=\linewidth]{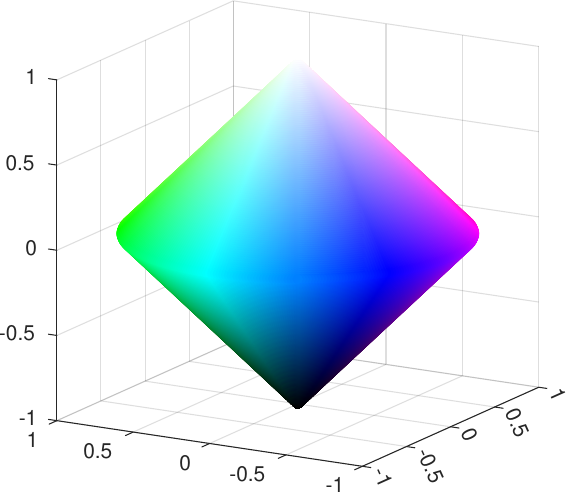}
    \end{minipage}
    \hspace{0.5cm}
    \begin{minipage}{0.4\linewidth}
        \includegraphics[width=\linewidth]{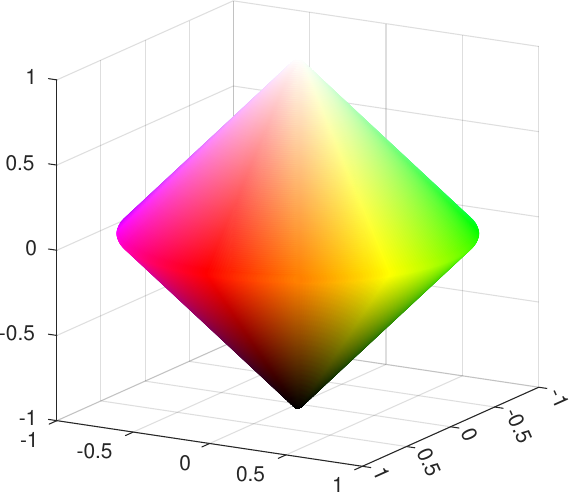}
    \end{minipage}
    \caption{The HCL bi-cone according to \cite{BurgethKleefeld}.}
    \label{fig:bicone}
\end{figure}

\subsection{The Loewner Order and the Decision of a Minimiser}

In the absence of a conventional ordering of the elements in either $\mathbb{R}^3$ or $\mathbb{R}^{2 \times 2}$, it becomes necessary to define the relative positions of two elements with respect to one another. In order to respond to this query, it is necessary to utilise a more lenient interpretation of the ordering relation, namely that of a semi-order. For this reason, we resort to the Loewner order and the colour morphological processing based on this, which was already presented in the work \cite{BurgethKleefeld} by Burgeth and Kleefeld. 

\begin{definition}
    We define the \textbf{set of symmetric matrices} as
    \begin{align}
        \textup{Sym}(n) := \{ \boldsymbol{A} = (a_{ij})_{i,j = 1, \dots, n} \in \mathbb R^{n \times n} : a_{ij} = a_{ji} \quad \forall i,j = 1, \dots, n \}, \quad n \in  \mathbb N,
        \label{symm_mat}
    \end{align}
    and the \textbf{set of positive semi-definite symmetric matrices} as 
    \begin{align}
        \textup{Sym}_+(n) := \{ \boldsymbol{A} \in \textup{Sym}(n) : \boldsymbol{x}^\textup T \boldsymbol{A} \boldsymbol{x} \geq 0 \quad \forall \boldsymbol{x} \in \mathbb R^n \textup{ with } \boldsymbol{x} \neq \boldsymbol{0} \}, \quad n  \in \mathbb N.
        \label{pos_semidef_mat}
        \end{align}
\end{definition}

\noindent One can see that $\textup{Sym}_+(n)$ is a convex cone, cf. \cite{convex_Analysis}. Thus, this cone induces a partial order in the space of symmetric matrices, which we define as follows.

\begin{definition}
    Let $\boldsymbol{A}, \boldsymbol{B} \in \textup{Sym}(n)$, $n \in \mathbb N$. We define the \textbf{Loewner semi-order} $\geq_\textup L$ as follows
    \begin{align}
        \boldsymbol{A} \geq_\textup L \boldsymbol{B} :\Longleftrightarrow \boldsymbol{A} - \boldsymbol{B} \in \textup{Sym}_+(n).
        \label{loewner_order}
    \end{align}
\end{definition}

One of the benefits of employing the Loewner order in conjunction with the HCL bi-cone is that the Cartesian coordinates of a given bi-cone vector can be utilised to represent each colour matrix $\boldsymbol{X}$ as a cone of a specified height within the bi-cone. In this representation, the radius of the cone of $\boldsymbol{X}$ is given by $\frac{1}{\sqrt{2}} \tr \boldsymbol{X}$. Therefore, a colour is larger than another colour in the Loewner sense if the cone of the larger colour contains the cone of the smaller colour as a subset. Since the radius is equal to the height of the resulting cones, we only need to find a larger base-circle for the corresponding cones, see Figure \ref{fig:Loewner_order_cones} for a visualising example. For more details,  we refer to the paper \cite{BurgethKleefeld} by Burgeth and Kleefeld.

\begin{figure}
    \centering
    \begin{minipage}{0.31\linewidth}
        \includegraphics[width=\linewidth]{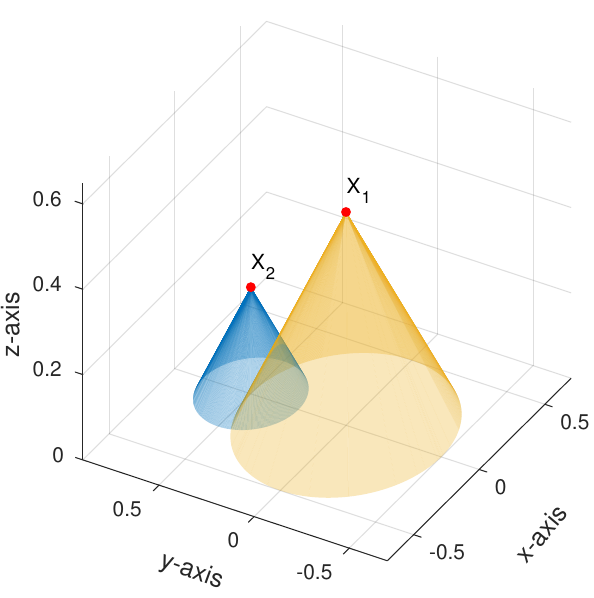}
    \end{minipage}
    \hspace{0.1cm}
    \begin{minipage}{0.31\linewidth}
        \includegraphics[width=\linewidth]{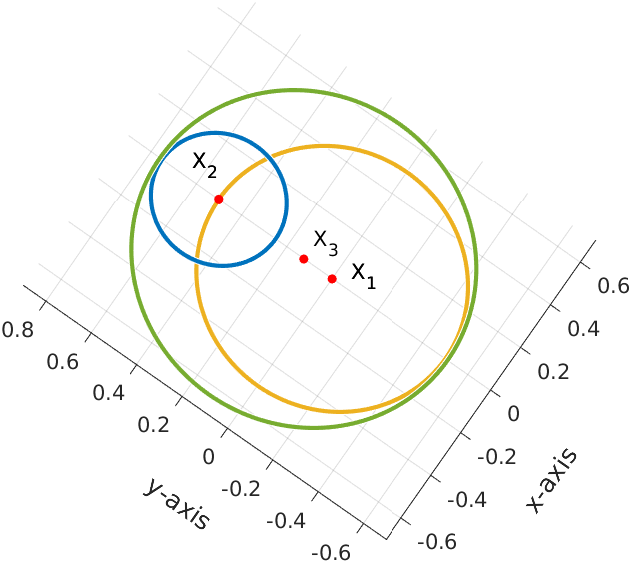}
    \end{minipage}
    \hspace{0.1cm}
    \begin{minipage}{0.31\linewidth}
        \includegraphics[width=\linewidth]{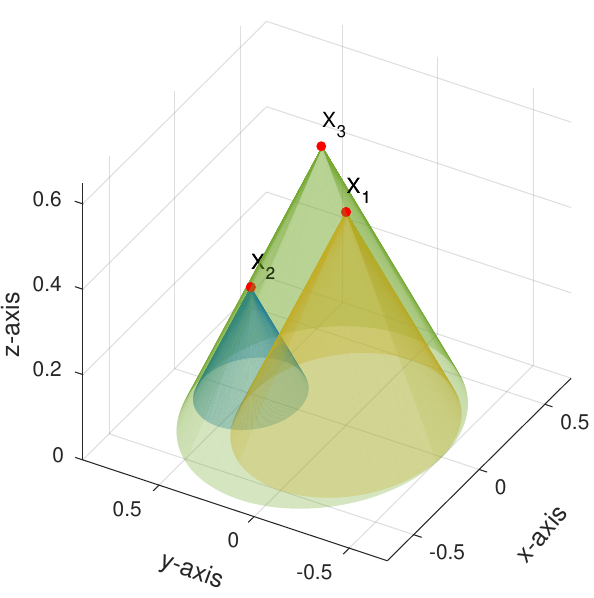}
    \end{minipage}
    \caption{Visualisation of the Loewner ordering in the HCL bi-cone for the three colour matrices $\boldsymbol{X}_i$, $i = 1, 2, 3$, with $\boldsymbol{X}_3 \geq_{\textup L} \boldsymbol{X}_1$ and $\boldsymbol{X}_3 \geq_{\textup L} \boldsymbol{X}_2$. 
    {\bf From left to right:} The cone representation of a yellow colour $\boldsymbol X_1$ and a cyan-blue colour $\boldsymbol X_2$, the base circles of $\boldsymbol X_1$ and $\boldsymbol X_2$ with a third base circle of a green colour $\boldsymbol X_3$ that encompasses both of them, the cone representation of all three colours.
    }
    \label{fig:Loewner_order_cones}
\end{figure}

The problem with this semi-order is that it is not a lattice order \cite{BorweinLewis} and therefore it is not possible to find a unique maximum or minimum. To get around this problem, one needs another property to select a uniquely determined maximum from the convex set of symmetric matrices $\mathcal{U}(\mathcal{X})$ that are upper bounds in the Loewner sense for the multi-set $\mathcal{X} = \{ \boldsymbol X_1, \boldsymbol X_2, \dots, \boldsymbol X_n \}$ of given symmetric real $2 \times 2$ data matrices with
\begin{align*}
    \mathcal{U}(\mathcal{X}) := \left\{ \boldsymbol Y \in \text{Sym}(2) : \boldsymbol X \leq_\textup L \boldsymbol Y \quad \forall \boldsymbol X \in \mathcal{X} \right\},
\end{align*}
see \cite{WelkQuantile} for this. In particular, $\mathcal{U}(\mathcal{X})$ is a convex set, since for any $\boldsymbol A, \boldsymbol B \in \mathcal{U}(\mathcal{X})$ one has
\begin{align*}
    \alpha \boldsymbol{A} + (1-\alpha) \boldsymbol{B} \geq_\textup L \alpha \boldsymbol{X} + (1-\alpha) \boldsymbol{X} = \boldsymbol{X} \quad \forall \boldsymbol{X} \in \mathcal{X}, \, \alpha \in [0,1]
\end{align*}
and the linear combination of two symmetric matrices always result in a symmetric matrix. For representation of the previously mentioned property, we use the function $\varphi : \mathcal{U}(\mathcal{X}) \rightarrow \mathbb{R}$, which should be convex on the set $\mathcal{U}(\mathcal{X})$ and Loewner-monotone, i.e.
\begin{align*}
    \varphi(\boldsymbol A) \leq \varphi(\boldsymbol B) \quad \Longleftrightarrow \quad \boldsymbol A \leq_\textup L \boldsymbol B. 
\end{align*}
To clarify this terminology of convexity, we introduce the following definition.

\begin{definition}
    Let $V_1, V_2$ be two real-valued vector spaces, $M \subset V_1$ a convex set and $K \subset V_2$ an order cone, i.e. $K := \{ \boldsymbol{x} \in V_2 : \boldsymbol{x} \geq 0 \}$ on the ordered vector space $(V_2, \geq)$. A function $\boldsymbol{f}: M \rightarrow V_2$ is called \textbf{convex on the set $M$} if and only if
    \begin{align}
        \alpha \boldsymbol{f}(\boldsymbol{x}) + (1 - \alpha) \boldsymbol{f}(\boldsymbol{y}) - \boldsymbol{f}(\alpha \boldsymbol{x} + (1-\alpha) \boldsymbol{y}) \in K \quad \forall \boldsymbol{x},\boldsymbol{y} \in M, \, \alpha\in [0,1].
        \label{convexity}
    \end{align}
\end{definition}

\noindent Furthermore, $\varphi$ should have a unique minimiser in $\mathcal{U}(\mathcal{X})$, thus we can define the $\varphi$-supremum of $\mathcal{X}$ as said minimiser:
\begin{align}
    \Sup_{\varphi}(\mathcal{X}) := \argmin_{\boldsymbol Y \in \mathcal{U}(\mathcal{X})} \varphi(\boldsymbol Y).
    \label{phi}
\end{align}
The matrix supremum introduced in the paper \cite{morph_op_mat_im} is based on the calculation of the trace and is therefore also called trace-supremum. That is, one has $\varphi(\boldsymbol{Y}) = \trace \boldsymbol{Y}$ and we get as supremum:
\begin{align*}
    \Sup_{\trace}(\mathcal{X}) := \argmin_{\boldsymbol Y \in \mathcal{U}(\mathcal{X})} \trace \boldsymbol Y.
\end{align*}
Based on the corresponding norms, the Frobenius supremum 
\begin{align*}
    \Sup_{2}(\mathcal{X}) := \argmin_{\boldsymbol Y \in \mathcal{U}(\mathcal{X})} \sum_{\boldsymbol X \in \mathcal{X}} \norm{\boldsymbol Y - \boldsymbol X}_2
\end{align*}
and the spectral supremum 
\begin{align*}
    \Sup_{\infty}(\mathcal{X}) := \argmin_{\boldsymbol Y \in \mathcal{U}(\mathcal{X})} \sum_{\boldsymbol X \in \mathcal{X}} \abs{\lambda_1(\boldsymbol Y - \boldsymbol X)},
\end{align*}
where $\lambda_1(\boldsymbol A)$ denotes the largest eigenvalue of $\boldsymbol A$, were derived in \cite{WelkQuantile}. At this point it should be noted that, in the case of positive semi-definite matrices, all three norms are Schatten norms $\norm{\,\cdot\,}_p$ for $p \in \{ 1, 2, \infty \}$.

In contrast to the above approaches utilising norms, which compare matrices that are upper bounds, we will adopt a more direct approach. To accomplish this, we will employ an approximation of the maximum function by Maslov \cite{Maslov} as a matrix-valued function to directly calculate a matrix that acts as an upper bound. The subsequent section aims to present a detailed characterisation of this approach.

\section{Characterisation of the Log-Exp-Supremum}

In this section we will construct a characterisation for the log-exp-supremum that depends solely on the input data, albeit not continuously. This characterisation is based on the spectral decomposition of the matrices under consideration and some properties of symmetric $2 \times 2$ matrices and the Rayleigh product. Given that this characterisation is based on spectral decomposition, this section will be divided into three parts. The first part will consider the case where there is a unique largest eigenvalue among the input data. The second part will consider the case where there is no unique largest eigenvalue. The third and final part will consider a general case that includes the other two cases and some associated properties. In particular, the log-exp-supremum is transitive, a property that will be demonstrated at the conclusion of this section.

\begin{definition}
    We define for a multi-set $\mathcal{X} = \{\boldsymbol X_1,\dots , \boldsymbol X_n\}$, $n \in \mathbb{N}$, of symmetric real $2 \times 2$ matrices the \textbf{log-exp-supremum (LES)} as 
\begin{align}
    \boldsymbol S := \Sup_{\textup{LE}}(\mathcal{X}) := \lim_{m \rightarrow \infty} \left( \frac{1}{m} \log \sum_{i=1}^n \exp(m \boldsymbol X_i) \right).
    \label{LES}
\end{align}
\end{definition}

\begin{remark}
    The log-exp-infimum follows by duality:
    \begin{align}
        \textup{Inf}_{\textup{LE}}(\mathcal{X}) := - \lim_{m \rightarrow \infty} \left( \frac{1}{m} \log \sum_{i=1}^n \exp(- m \boldsymbol X_i) \right).
        \label{LEI}
    \end{align}
\end{remark}

\noindent Furthermore, we explain the spectral decomposition of $\boldsymbol X_i$ by 
\begin{align}
    \begin{split}
        &\boldsymbol X_i = \lambda_i \boldsymbol u_i \boldsymbol u_i^\textup{T} + \mu_i \boldsymbol v_i \boldsymbol v_i^\textup{T}, \quad \lambda_i \geq \mu_i, \quad \langle \boldsymbol u_i, \boldsymbol v_i \rangle = 0, \quad \abs{\boldsymbol u_i} = 1 = \abs{\boldsymbol v_i}, \\
        &\boldsymbol u_i = (c_i,s_i)^\textup{T} , \quad \boldsymbol v_i = (-s_i,c_i)^\textup{T}, \quad  c_i = \cos(\varphi_i), \quad s_i = \sin(\varphi_i),\\
        &\varphi_i \in \left[- \frac{\pi}{2}, \frac{\pi}{2}\right], \quad i = 1,\dots,n,
    \end{split}
    \label{spectral_decomposition}
\end{align}
where $\lambda_i,\mu_i \in \mathbb{R}$ are the eigenvalues of $\boldsymbol X_i$ and $\boldsymbol u_i \perp \boldsymbol v_i$ are the normalised eigenvectors to the corresponding eigenvalues.

\subsection{LES for a unique largest Eigenvalue}

We will now assume that 
\begin{align}
    \text{$\lambda_1$ is the unique largest eigenvalue among all eigenvalues of $\mathcal{X}$}
    \label{lambda1_unique}
\end{align}
 and further that without loss of generality 
 \begin{align}
     \text{the eigenvectors of $\boldsymbol X_1$ have $\varphi_1 = 0$,}
     \label{rotation}
 \end{align}
i.e. they are axis-aligned and fulfil $\boldsymbol u_1 = \boldsymbol e_1 := (1,0)^\textup{T}$ and $\boldsymbol v_1 = \boldsymbol e_2 := (0,1)^\textup{T}$. Using these assumptions, we calculate the matrix exponential according to the general application of functions to diagonalisable matrices, see \cite{matrix_algebra}:
\begin{align*}
    \exp(m \boldsymbol X_i) &= \exp\left( m \lambda_i
    \begin{pmatrix}
        c_i^2 \hspace{0.5em}\hspace{0.5em}& c_i s_i \\
        c_i s_i \hspace{0.5em}& s_i^2
    \end{pmatrix}
    + m \mu_i 
    \begin{pmatrix}
        s_i^2 \hspace{0.5em}& -c_i s_i \\
        -c_i s_i \hspace{0.5em}& c_i^2
    \end{pmatrix}
    \right) \\
    &= 
    \begin{pmatrix}
        c_i^2 \mathrm{e}^{m \lambda_i} + s_i^2 \mathrm{e}^{m \mu_i} \hspace{0.5em}& c_i s_i \left( \mathrm{e}^{m\lambda_i} - \mathrm{e}^{m \mu_i} \right) \\
        c_i s_i \left( \mathrm{e}^{m\lambda_i} - \mathrm{e}^{m \mu_i} \right) \hspace{0.5em}& s_i^2 \mathrm{e}^{m \lambda_i} + c_i^2 \mathrm{e}^{m \mu_i}
    \end{pmatrix}
\end{align*}
and the summation over all $i = 1,\dots,n$ leads to
\begin{align*}
    \boldsymbol E_m &:= \sum_{i=1}^n \exp(m \boldsymbol X_i) \\
    &= 
    \begin{pmatrix}
        \mathrm{e}^{m \lambda_1} + \sum_{i=2}^n \left( c_i^2 \mathrm{e}^{m \lambda_i} + s_i^2 \mathrm{e}^{m \mu_i} \right) \hspace{0.5em}& \sum_{i=2}^n c_i s_i \left( \mathrm{e}^{m\lambda_i} - \mathrm{e}^{m \mu_i} \right) \\
        \sum_{i=2}^n c_i s_i \left( \mathrm{e}^{m\lambda_i} - \mathrm{e}^{m \mu_i} \right) \hspace{0.5em}& \mathrm{e}^{m \mu_1} + \sum_{i=2}^n \left( s_i^2 \mathrm{e}^{m \lambda_i} + c_i^2 \mathrm{e}^{m \mu_i} \right)
    \end{pmatrix}.
\end{align*}
Next, we determine the major eigenvector for this matrix when $m$ tends to infinity.

\begin{lemma}
    Let the conditions \eqref{lambda1_unique} and \eqref{rotation} be fulfilled according to \eqref{spectral_decomposition}. Then $\boldsymbol u_1 = (1,0)^\textup{T}$ represents the major eigenvector of $\boldsymbol E_m$ for $m \rightarrow \infty $.
\end{lemma}
\begin{proof}
    We declare the Rayleigh product as
    \begin{align*}
        R_{\boldsymbol E_m}(\boldsymbol w) := \langle \boldsymbol w, \boldsymbol E_m \boldsymbol w \rangle, \quad \boldsymbol w \in \mathbb{R}^2,
    \end{align*}
    and consider it for vectors 
    \begin{align*}
        \boldsymbol w_\varepsilon = \frac{1}{\sqrt{1 + \varepsilon^2}} 
        \begin{pmatrix}
            1 \\
            \varepsilon
        \end{pmatrix}, \quad \abs{\varepsilon} \ll 1,
    \end{align*}
    which vary in the second component with $\varepsilon$ around $\boldsymbol u_1$. We determine
    \begin{align}
        R_{\boldsymbol E_m}(\boldsymbol w_\varepsilon) = \frac{1}{1 + \varepsilon^2} \left( \mathrm{e}^{m \lambda_1} + \varepsilon^2 \mathrm{e}^{m \mu_1} + \sum_{i=2}^n \left( (c_i + \varepsilon s_i)^2 \mathrm{e}^{m \lambda_i} + (s_i - \varepsilon c_i)^2 \mathrm{e}^{m \mu_i} \right) \right).
        \label{M_EW_E}
    \end{align}
    If $m \rightarrow \infty$, we observe
    \begin{align}
        \lim_{m \rightarrow \infty} \frac{\varepsilon^2 \mathrm{e}^{m \mu_1} + \sum_{i=2}^n \left( (c_i + \varepsilon s_i)^2 \mathrm{e}^{m \lambda_i} + (s_i - \varepsilon c_i)^2 \mathrm{e}^{m \mu_i} \right)}{\mathrm{e}^{m \lambda_1}} = 0,
        \label{lim_ev_1}
    \end{align}
    because we assumed that $\lambda_1$ is the unique largest eigenvalue and $\abs{\varepsilon} \ll 1$. By considering 
    \begin{align*}
        R_{\boldsymbol E_m}(\boldsymbol w_0) = \mathrm{e}^{m \lambda_1} + \sum_{i=2}^n \left( c_i^2 \mathrm{e}^{m \lambda_i} + s_i^2 \mathrm{e}^{m \mu_i} \right),
    \end{align*}
    we obtain
    \begin{align*}
        &\lim_{m \rightarrow \infty} \frac{R_{\boldsymbol E_m}(\boldsymbol w_\varepsilon)}{R_{\boldsymbol E_m}(\boldsymbol w_0)} \\
        & \quad = \frac{1}{1 + \varepsilon^2} \lim_{m \rightarrow \infty} \frac{\mathrm{e}^{m \lambda_1} + \varepsilon^2 \mathrm{e}^{m \mu_1} + \sum_{i=2}^n \left( (c_i + \varepsilon s_i)^2 \mathrm{e}^{m \lambda_i} + (s_i - \varepsilon c_i)^2 \mathrm{e}^{m \mu_i} \right)}{\mathrm{e}^{m \lambda_1} + \sum_{i=2}^n \left( c_i^2 \mathrm{e}^{m \lambda_i} + s_i^2 \mathrm{e}^{m \mu_i} \right)} \\
        & \quad = \frac{1}{1 + \varepsilon^2} \lim_{m \rightarrow \infty} \Bigg(  1 \\
        &\qquad + \underbrace{ \frac{\varepsilon^2 \mathrm{e}^{m \mu_1} + \sum_{i=2}^n \left( (2 \varepsilon c_i s_i + \varepsilon^2 s_i^2) \mathrm{e}^{m \lambda_i} + (-2 \varepsilon c_i s_i + \varepsilon^2 c_i^2) \mathrm{e}^{m \mu_i} \right)}{\mathrm{e}^{m \lambda_1} + \sum_{i=2}^n \left( c_i^2 \mathrm{e}^{m \lambda_i} + s_i^2 \mathrm{e}^{m \mu_i} \right)}}_{=: f(m) = \frac{f_1(m)}{f_2(m)}} \Bigg).
    \end{align*}
    We know that $f(m) \geq 0$ for all $m > 0$, since $\lambda_i \geq \mu_i$ for $i = 2,\dots,n$. Furthermore, we set 
    \begin{align*}
        g(m) := \sum_{i=2}^n \left( (c_i + \varepsilon s_i)^2 \mathrm{e}^{m \lambda_i} + (s_i - \varepsilon c_i)^2 \mathrm{e}^{m \mu_i} \right)
    \end{align*}
    and estimate $f_1(m) \leq \varepsilon^2 \mathrm{e}^{m \mu_1} + g(m)$. We therefore have the following estimation:
    \begin{align*}
        0 \leq \lim_{m \rightarrow \infty} f(m) = \lim_{m \rightarrow \infty} \frac{f_1(m)}{f_2(m)} \leq \lim_{m \rightarrow \infty} \frac{\varepsilon^2 \mathrm{e}^{m \mu_1} + g(m)}{f_2(m)} \overset{\eqref{lim_ev_1}}{=} 0,
    \end{align*}
    which leads to
    \begin{align*}
        \lim_{m \rightarrow \infty} \frac{R_{\boldsymbol E_m}(\boldsymbol w_\varepsilon)}{R_{\boldsymbol E_m}(\boldsymbol w_0)} = \frac{1}{1 + \varepsilon^2} (1 + 0) = \frac{1}{1 + \varepsilon^2}.
    \end{align*}
    This means that the function $\lim\limits_{m \rightarrow \infty} \frac{R_{\boldsymbol E_m}(\boldsymbol w_\varepsilon)}{R_{\boldsymbol E_m}(\boldsymbol w_0)}$ has a local maximum at $\varepsilon = 0$ and thus concludes the proof. 
    \hfill $\square$
\end{proof}

To interpret this lemma, we declare the major eigenvector of $\boldsymbol E_m$ as $\Bar{\boldsymbol u}_m$. Then we obtain from the lemma the equation $\lim\limits_{m \rightarrow 0} \Bar{\boldsymbol u}_m = \boldsymbol u_1$ and since the eigenvectors of $\frac{1}{m} \log \boldsymbol E_m$ are the same as those of $\boldsymbol E_m$, it also follows that $\boldsymbol u_1$ is the major eigenvector of $\boldsymbol S$. Since all considered matrices are $2 \times 2$ matrices, we can deduce the second normalised eigenvector of $\boldsymbol S$ as the vector that is orthogonal to the first eigenvector $\boldsymbol u_1 = (1,0)^\textup{T}$, namely $\boldsymbol v_1 = (0,1)^\textup{T}$. As the next step we will determine the eigenvalue corresponding to $\boldsymbol u_1$.

\begin{lemma}
    Let the conditions of Lemma 1 be fulfilled. Then the eigenvector $\boldsymbol u_1$ of $\boldsymbol S$ according to \eqref{LES} has the corresponding eigenvalue $\lambda_1$.
\end{lemma}
\begin{proof}
    Since the largest eigenvalue of $\boldsymbol E_m$ is just expressed by the value of the Rayleigh product at the point where it takes its maximum, i.e. at $\boldsymbol w_\varepsilon$, as we have established in the proof of Lemma 2, we obtain as major eigenvalue \eqref{M_EW_E}. This means that the major eigenvalue of $\frac{1}{m} \log \boldsymbol E_m$ is
    \begin{align*}
        &\frac{1}{m} \log R_{\boldsymbol E_m}(\boldsymbol w_\varepsilon) \\
        &= \frac{1}{m} \log \bigg[ \frac{1}{1 + \varepsilon^2} \Big( \mathrm{e}^{m \lambda_1} + \varepsilon^2 \mathrm{e}^{m \mu_1} + \sum_{i=2}^n \big( (c_i + \varepsilon s_i)^2 \mathrm{e}^{m \lambda_i} + (s_i - \varepsilon c_i)^2 \mathrm{e}^{m \mu_i} \big) \Big) \bigg] \\
        &= \frac{1}{m} \log \bigg[ \frac{\mathrm{e}^{m \lambda_1}}{1 + \varepsilon^2} \Big( 1 + \varepsilon^2 \mathrm{e}^{m (\mu_1 - \lambda_1)} + \sum_{i=2}^n \big( (c_i + \varepsilon s_i)^2 \mathrm{e}^{m (\lambda_i - \lambda_1)} \\
        & \quad + (s_i - \varepsilon c_i)^2 \mathrm{e}^{m (\mu_i - \lambda_1)} \big) \Big) \bigg] \\
        &= \frac{1}{m} \log \Big( 1 + \underbrace{\varepsilon^2 \mathrm{e}^{m (\mu_1 - \lambda_1)}}_{\rightarrow 0 ~ (m \rightarrow \infty)} + \sum_{i=2}^n \underbrace{\big( (c_i + \varepsilon s_i)^2 \mathrm{e}^{m (\lambda_i - \lambda_1)} + (s_i - \varepsilon c_i)^2 \mathrm{e}^{m (\mu_i - \lambda_1)} \big)}_{\rightarrow 0 ~ (m \rightarrow \infty)} \Big) \\
        &\quad + \frac{1}{m} \log \mathrm{e}^{m \lambda_1} - \frac{1}{m} \log(1 + \varepsilon^2),
    \end{align*}
    where the second and third term in the first log function vanish in the limit because $\mu_1,\lambda_i,\mu_i$ for $i = 2,\dots,n$ are smaller than $\lambda_1$. Thus the major eigenvalue of $\boldsymbol S$ is just 
    \begin{align*}
        \lim_{m \rightarrow \infty} \left( \frac{1}{m} \log R_{\boldsymbol E_m}(\boldsymbol w_\varepsilon) \right) = \lim_{m \rightarrow \infty} \left( \frac{1}{m} \log \mathrm{e}^{m \lambda_1} \right) = \lambda_1.
    \end{align*}
    \hfill $\square$
\end{proof}

The final step in the characterisation of $\boldsymbol S$ is to calculate the second eigenvalue corresponding to $\boldsymbol v_1$, but for this we need a refinement of our calculations in the proof of Lemma 1. To achieve this, we establish the following lemma:

\begin{lemma}
    Let the conditions of Lemma 1 be fulfilled and additionally $\varepsilon(m)$ be the value of $\varepsilon$ for which $\boldsymbol v_m = \boldsymbol w_\varepsilon$ holds for every $m > 0$. Then $\lim\limits_{m \rightarrow \infty} \varepsilon(m) = 0$ with $\varepsilon(m) \sim \mathrm{e}^{m(\lambda - \lambda_1)}$, where $\lambda = \max(\mu_1,\lambda_2,\dots,\lambda_n)$. 
\end{lemma}
\begin{proof}
    In the proof of Lemma 1 we have already seen that $\varepsilon \rightarrow 0$ for $m \rightarrow \infty$ and that the Rayleigh product fulfils 
    \begin{align*}
        R(\varepsilon) := R_{\boldsymbol E_m}(\boldsymbol w_\varepsilon) = \langle \boldsymbol w_\varepsilon, \boldsymbol E_m \boldsymbol w_\varepsilon \rangle \in C^\infty(\mathbb{R}).
    \end{align*}
    Since the critical point is at $\varepsilon \rightarrow 0$, we want to approximate this point by a Taylor approximation of order two:
    \begin{align}
        T_2R(\varepsilon,0) = R(0) + R'(0) \varepsilon + \frac{1}{2} R''(0) \varepsilon^2.
        \label{Taylor}
    \end{align}
    For this, we determine
    \begin{alignat*}{3}
        &R(0) &&= \mathrm{e}^{m \lambda_1} + \sum_{i=2}^n \left( c_i^2 \mathrm{e}^{m \lambda_i} + s_i^2 \mathrm{e}^{m \mu_i} \right), \\
        &R'(\varepsilon) &&= - \frac{2\varepsilon}{(1+\varepsilon^2)^2} \Big(R(\varepsilon)(1+\varepsilon^2) \Big) + \frac{1}{1+\varepsilon^2} \bigg( 2\varepsilon \mathrm{e}^{m \mu_1} \\
        & && \quad + \sum_{i=2}^n \Big( \big( 2 s_i c_i + 2 \varepsilon s_i^2 \big) \mathrm{e}^{m \lambda_i} + \big( - 2 s_i c_i + 2 \varepsilon c_i^2 \big) \mathrm{e}^{m \mu_i} \Big) \bigg) \\
        & &&= - \frac{2\varepsilon}{1+\varepsilon^2} R(\varepsilon) + \frac{2\varepsilon}{1+\varepsilon^2} \mathrm{e}^{m \mu_1} + \sum_{i=2}^n \frac{2(c_i + \varepsilon s_i)s_i}{1+\varepsilon^2} \mathrm{e}^{m \lambda_i} \\
        & && \quad + \sum_{i=2}^n \frac{2(-s_i + \varepsilon c_i) c_i}{1+\varepsilon^2} \mathrm{e}^{m \mu_i}, \\
        &R'(0) &&= \sum_{i=2}^n 2c_i s_i \mathrm{e}^{m \lambda_i} - \sum_{i=2}^n 2 s_i c_i \mathrm{e}^{m \mu_i} = 2 \sum_{i=2}^n c_i s_i \left( \mathrm{e}^{m \lambda_i} - \mathrm{e}^{m \mu_i} \right), \\
    \end{alignat*}
    \begin{alignat*}{3}
        &R''(\varepsilon) &&= \frac{-2 (1 + \varepsilon^2) + 4 \varepsilon^2}{(1+\varepsilon^2)^2} R(\varepsilon) - \frac{2\varepsilon}{1+\varepsilon^2} R'(\varepsilon) \\
        & && \quad + \frac{-2(1+\varepsilon^2) + 4\varepsilon^2}{(1+\varepsilon^2)^2} \mathrm{e}^{m \mu_1} + \sum_{i=2}^n \frac{2 s_i^2 (1+\varepsilon^2) - 4\varepsilon (c_i + \varepsilon s_i) s_i}{(1+\varepsilon^2)^2} \mathrm{e}^{m \lambda_i} \\
        & && \quad + \sum_{i=2}^n \frac{2 c_i^2 (1+\varepsilon^2) - 4\varepsilon (-s_i + \varepsilon c_i) c_i}{(1+\varepsilon^2)^2} \mathrm{e}^{m \mu_i} \\
        & &&= - \frac{2(1-\varepsilon^2)}{(1+\varepsilon^2)^2} R(\varepsilon) - \frac{2\varepsilon}{1+\varepsilon^2} R'(\varepsilon) + \frac{2(1-\varepsilon^2)}{(1+\varepsilon^2)^2} \mathrm{e}^{m \mu_1} \\
        & && \quad + \sum_{i=2}^n \frac{2(1-\varepsilon^2) s_i^2 - 4\varepsilon c_i s_i}{(1+\varepsilon^2)^2} \mathrm{e}^{m \lambda_i} + \sum_{i=2}^n \frac{2(1-\varepsilon^2) c_i^2 + 4\varepsilon c_i s_i}{(1+\varepsilon^2)^2} \mathrm{e}^{m \mu_i}, \\
        &R''(0) &&= -2R(0) + 2 \mathrm{e}^{m \mu_1} + \sum_{i=2}^n 2 s_i^2 \mathrm{e}^{m \lambda_i} + \sum_{i=2}^n 2 c_i^2 \mathrm{e}^{m \mu_i} \\
        & &&= -2\mathrm{e}^{m \lambda_1} - 2 \sum_{i=2}^n \left( c_i^2 \mathrm{e}^{m \lambda_i} + s_i^2 \mathrm{e}^{m \mu_i} \right) + 2 \mathrm{e}^{m \mu_1} + 2 \sum_{i=2}^n \left( s_i^2 \mathrm{e}^{m \lambda_i} + c_i^2 \mathrm{e}^{m \mu_i} \right)\\
        & &&= -2\mathrm{e}^{m \lambda_1} + 2 \mathrm{e}^{m \mu_1} - 2 \sum_{i=2}^n \left( c_i^2 - s_i^2 \right) \mathrm{e}^{m \lambda_i} - 2 \sum_{i=2}^n \left( s_i^2 - c_i^2 \right) \mathrm{e}^{m \mu_i} 
    \end{alignat*}    
    and by setting $\lambda := \max(\mu_1,\lambda_2,\dots,\lambda_n)$, we obtain
    \begin{alignat*}{3}
        &R(0) &&= \mathrm{e}^{m \lambda_1} \left( 1 + \sum_{i=2}^n \left( c_i^2 \mathrm{e}^{m (\lambda_i - \lambda_1)} + s_i^2 \mathrm{e}^{m (\mu_i - \lambda_1)} \right) \right) \\
        & &&= \mathrm{e}^{m \lambda_1} \left( 1 + \mathcal{O} \left( \mathrm{e}^{m (\lambda - \lambda_1)} \right) \right), \\
        &R'(0) &&= \Theta\left( \mathrm{e}^{m \lambda} \right), \\
        &R''(0) &&= -2 \mathrm{e}^{m \lambda_1} \bigg( 1 - \mathrm{e}^{m (\mu_1 - \lambda_1)}+ \sum_{i=2}^n \left( c_i^2 - s_i^2 \right) \mathrm{e}^{m (\lambda_i - \lambda_1)} \\
        & && \quad + \sum_{i=2}^n \left( s_i^2 - c_i^2 \right) \mathrm{e}^{m (\mu_i - \lambda_1)} \bigg) \\
        & &&=  -2 \mathrm{e}^{m \lambda_1} \left( 1 + \mathcal{O} \left( \mathrm{e}^{m (\lambda - \lambda_1)} \right) \right).
    \end{alignat*}
    In this context, the symbols $\mathcal{O}$ and $\Theta$ represent the corresponding Landau symbols. The term $\mathcal{O}(\, \cdot\,)$ symbolises an upper bound for the asymptotic order, while the term $\Theta(\,\cdot\,)$ denotes the exact asymptotic order. We will now determine the maximum of \eqref{Taylor} by means of differentiation:
    \begin{align*}
        &R'(\varepsilon) = \dv{\varepsilon} \Big(T_2R(\varepsilon,0) + \mathcal{O} \left(\varepsilon^3\right)\Big) = R'(0) + R''(0) \varepsilon + 3 \mathcal O\left(\varepsilon^2\right) \overset{!}{=} 0 
    \end{align*}
    \begin{align*}
        &\Longleftrightarrow \varepsilon = - \frac{R'(0) + 3 \mathcal O\left(\varepsilon^2\right)}{R''(0)} = - \frac{\Theta\left( \mathrm{e}^{m \lambda} \right) + 3 \mathcal O\left(\varepsilon^2\right)}{-2 \mathrm{e}^{m \lambda_1} \left( 1 + \mathcal{O} \left( \mathrm{e}^{m (\lambda - \lambda_1)} \right) \right)} \\
        &\Longleftrightarrow \varepsilon = \frac{1}{2 \Theta\left(\mathrm{e}^{m (\lambda_1 - \lambda)}\right) \left( 1 + \mathcal{O} \left( \mathrm{e}^{m (\lambda - \lambda_1)} \right) \right)} + \frac{3 \mathcal O \left( \varepsilon^2 \right)}{2 \mathrm{e}^{m \lambda_1} \left( 1 + \mathcal{O} \left( \mathrm{e}^{m (\lambda - \lambda_1)} \right) \right)} \\
        &\Longleftrightarrow \varepsilon = \frac{1}{\Theta\left(\mathrm{e}^{m (\lambda_1 - \lambda)}\right) +  \mathcal{O}(1)} + \frac{3 \mathcal O \left( \varepsilon^2 \right)}{2 \mathrm{e}^{m \lambda_1} \left( 1 + \mathcal{O} \left( \mathrm{e}^{- m (\lambda_1 - \lambda)} \right) \right)}.
    \end{align*}
    Since we can add a constant value to all eigenvalues and subtract it again later when we have formed the LES, we can assume without loss of generality that $\lambda_1 = 1$. If we let $m \rightarrow \infty$ then the first fraction will converge to zero and the second fraction will diverge in the denominator to infinity, because $\lambda_1 > \lambda$. This means, we have $\varepsilon \rightarrow 0$ for $m \rightarrow \infty$. Together with $R''(0) < 0$ it follows that $\varepsilon \rightarrow 0$ represents a maximum of $R(\varepsilon)$, which concludes this proof. \hfill $\square$
\end{proof}

With this enhancement of our previous calculation of Lemma 1 we are in the position to determine the second eigenvalue that is needed for the characterisation of $\boldsymbol S$. To this end, we note the next lemma.

\begin{lemma}
    Let the conditions of Lemma 3 be fulfilled. Then, the eigenvalue of $\boldsymbol S$ belonging to the eigenvector $\boldsymbol v_1$ is the largest eigenvalue of any of the matrices $\boldsymbol X_i$, $i=1,\dots,n$, whose eigenvector is not aligned with $\boldsymbol u_1$. 
\end{lemma}
\begin{proof}
    $1.)$ Analogous to the proof of Lemma 1, we start with the Rayleigh product of $\boldsymbol E_m$, but with an eigenvector that is perpendicular to our first ``test eigenvector'' $\boldsymbol w_\varepsilon$, namely $\overline{\boldsymbol w}_\varepsilon = \boldsymbol w_\varepsilon^\perp = \frac{1}{\sqrt{1+\varepsilon^2}} (- \varepsilon, 1)^\textup{T}$:
    \begin{align*}
        \overline{R}(\varepsilon) &= R_{\boldsymbol E_m}(\overline{\boldsymbol w}_\varepsilon) = \langle \overline{\boldsymbol w}_\varepsilon, \boldsymbol E_m \overline{\boldsymbol w}_\varepsilon \rangle \\
        &= \frac{1}{1+\varepsilon^2} \left( \varepsilon^2 \mathrm{e}^{m \lambda_1} + \mathrm{e}^{m \mu_1} + \sum_{i=2}^n \left( \left( s_i - \varepsilon c_i \right)^2 \mathrm{e}^{m \lambda_i} + \left( c_i + \varepsilon s_i \right)^2 \mathrm{e}^{m \mu_i} \right) \right).
    \end{align*}\vspace{0.2cm}
    
    \noindent 2.) Let us first assume that the second largest eigenvalue is $\lambda_2$ and $s_2 \neq 0$, so that the eigenvector $\boldsymbol{u}_2$ is not aligned with the first eigenvector $\boldsymbol u_1$. Then we can rewrite this equation as
    \begin{align*}
        \overline{R}(\varepsilon) &= \frac{\mathrm{e}^{m \lambda_2}}{1 + \varepsilon^2} \bigg( \big( s_2 - \varepsilon c_2 \big)^2 + \varepsilon^2 \mathrm{e}^{m (\lambda_1 - \lambda_2)} + \mathrm{e}^{m (\mu_1- \lambda_2)} + \big( c_2 + \varepsilon s_2 \big)^2 \mathrm{e}^{m (\mu_2- \lambda_2)} \\
        & \quad + \sum_{i=3}^n \Big( \big( s_i - \varepsilon c_i \big)^2 \mathrm{e}^{m (\lambda_i - \lambda_2)} + \big( c_i + \varepsilon s_i \big)^2 \mathrm{e}^{m (\mu_i - \lambda_2)} \Big) \bigg).
    \end{align*}
    If we only look at the terms in the big brackets, we find that only $\big( s_2 - \varepsilon c_2 \big)^2$ is constant with respect to $m$, while all the other terms, with the exception of $\varepsilon^2 \mathrm{e}^{m (\lambda_1 - \lambda_2)}$, involve exponential functions with negative multiples of $m$, which will approach zero for $m \rightarrow \infty$. We conclude from Lemma 3 that
    \begin{align*}
        \varepsilon \sim \mathrm{e}^{m(\lambda_2 - \lambda_1)} \Longleftrightarrow \varepsilon^{-1} \sim \mathrm{e}^{m(\lambda_1 - \lambda_2)} \Longleftrightarrow \varepsilon^2 \mathrm{e}^{m(\lambda_1 - \lambda_2)} \sim \varepsilon \sim \mathrm{e}^{m(\lambda_2 - \lambda_1)}
    \end{align*}
    and we also know that $\varepsilon \rightarrow 0$ for $m \rightarrow \infty$. So will $\varepsilon^2 \mathrm{e}^{m(\lambda_1 - \lambda_2)} \rightarrow 0$ for $m \rightarrow \infty$. The remaining terms we consider are
    \begin{align}
        \overline{R}(\varepsilon) = \frac{\mathrm{e}^{m \lambda_2}}{1 + \varepsilon^2} \left( \big( s_2 - \varepsilon c_2 \big)^2 + \mathcal{O}\big( \mathrm{e}^{m(\lambda_2 - \lambda_1)} + \mathrm{e}^{m(\mu - \lambda_2)} \big) \right),
        \label{leading terms}
    \end{align}
    where $\mu$ is the next largest eigenvalue after $\lambda_2$. From there, we calculate
    \begin{align*}
        \frac{1}{m} \log \overline{R}(\varepsilon) &= \lambda_2 - \frac{1}{m} \log (1 + \varepsilon^2) \\
        & \quad + \frac{1}{m} \log \left( \big( s_2 - \varepsilon c_2 \big)^2 + \mathcal{O}\big( \mathrm{e}^{m(\lambda_2 - \lambda_1)} + \mathrm{e}^{m(\mu - \lambda_2)} \big) \right) \\
        &= \lambda_2 - \frac{\log (1 + \varepsilon^2)}{m} + \frac{\log \big( (s_2 - \varepsilon c_2)^2 \big)}{m} \\
        & \quad + \frac{\log \Big( 1 + \mathcal{O}\left( \mathrm{e}^{m(\lambda_2 - \lambda_1)} + \mathrm{e}^{m(\mu - \lambda_2)} \right) \Big)}{m}
    \end{align*}
    and for $m \rightarrow \infty$ we obtain that $\frac{1}{m} \log \overline{R}(\varepsilon) \rightarrow  \lambda_2$, since $\varepsilon$ goes exponentially to 0 for $m \rightarrow \infty$. This proves this case. \vspace{0.2cm}

    \noindent 3.) Let us now assume that $\mu_1$ is the second largest eigenvalue. In this case, the results of Lemma 3 would be obtained with $\lambda = \mu_1$, and the same procedure as that employed in the current proof would be followed, but with the exclusion of the $\mu_1$ term in place of the $\lambda_2$ term. Thereby, we would achieve
    \begin{align*}
        \overline{R}(\varepsilon) = \frac{\mathrm{e}^{m \mu_1}}{1 + \varepsilon^2} \left( 1 + \mathcal{O}\big( \mathrm{e}^{m(\mu_1 - \lambda_1)} + \mathrm{e}^{m(\mu - \mu_1)} \big) \right)
    \end{align*}
    for the leading terms in \eqref{leading terms}, where $\mu$ is the next largest eigenvalue after $\mu_1$. The rest is done analogously. \vspace{0.2cm}

    \noindent 4.) Now we assume that $\lambda_2$ has the same eigenvector as $\lambda_1$. Then it follows that $s_2 = 0$ and therefore $\big( s_2 - \varepsilon c_2 \big)^2 \rightarrow 0$ for $m \rightarrow \infty$ in equation \eqref{leading terms}. Because of that we would extract the next smaller eigenvalue and repeat the above calculation until we find an eigenvalue whose eigenvector is not $\boldsymbol u_1$. Then we would do the same proof with the term of this eigenvalue instead the $\lambda_2$ term. \vspace{0.2cm}

    \noindent 5.) The final case would be that the second largest eigenvalue $\lambda_2$ is not unique. This would result in further terms like $\big( s_2 - \varepsilon c_2 \big)^2$ in \eqref{leading terms} for the other eigenvalues equal to $\lambda_2$. Since these terms remain $\mathcal{O}(1)$, they will not change the result of the proof. \hfill $\square$
\end{proof}

\noindent We summarise our previous findings on our first theorem:

\begin{theorem}
    Let $\mathcal{X} = \{\boldsymbol X_1,\dots , \boldsymbol X_n\}$, $n \in \mathbb{N}$, be a multi-set of symmetric real $2 \times 2$ matrices and $\lambda_1$ be the unique largest eigenvalue of all matrices of $\mathcal{X}$ with the corresponding normalised eigenvector $\boldsymbol u_1$. Then the log-exp-supremum of $\mathcal{X}$ has the representation
    \begin{align}
        \boldsymbol S := \lim_{m \rightarrow \infty} \left( \frac{1}{m} \log \sum_{i=1}^n \exp(m \boldsymbol X_i) \right) = \lambda_1 \boldsymbol u_1 \boldsymbol u_1^\textup{T} + \mu_* \boldsymbol v_1 \boldsymbol v_1^\textup{T},
        \label{LES-unique}
    \end{align}
    where $\mu_*$ is the next largest eigenvalue among all matrices of $\mathcal{X}$ whose normalised eigenvector $\boldsymbol{v}_*$ is not aligned with the eigenvector $\boldsymbol{u}_1$ and $\boldsymbol v_1$ is the normalised eigenvector perpendicular to $\boldsymbol u_1$.
\end{theorem}
\begin{proof} 
    A rotation matrix $\boldsymbol{R}_{-\varphi_1}$ is employed to rotate all eigenvectors in $\mathcal{X}$, with the rotation angle represented by $-\varphi_1$. This results in the transformation of $\boldsymbol{X}_1$ into a diagonal matrix with diagonal entries $\lambda_1$ and $\mu_1$.
    Thereby, assumption \eqref{rotation} is fulfilled and we can apply Lemma 1--4 to the rotated input matrices $\mathcal{X}_{-\varphi_1}$. For the LES, we get
    \begin{align*}
        \boldsymbol{S}_{-\varphi_1} := \Sup_{\textup{LE}}(\mathcal{X}_{-\varphi_1}) = \lambda_1 \boldsymbol{e}_1 \boldsymbol{e}_1^{\textup T} + \mu_* \boldsymbol{e}_2 \boldsymbol{e}_2^{\textup T},
    \end{align*}
    where $\mu_*$ is the next largest eigenvalue among all matrices of $\mathcal{X}$ whose normalised eigenvector $\boldsymbol{v}_*$ is not aligned with the eigenvector $\boldsymbol{e}_1 = (1, 0)^\textup T$, and $\boldsymbol e_2 = (0, 1)^\textup T$. We achieve the LES of the original $\mathcal{X}$ by rotating all of the eigenvectors back with $\boldsymbol{R}_{\varphi_1}$:
    \begin{align*}
        \boldsymbol{S} &= \boldsymbol R_{\varphi_1} \boldsymbol{S}_{-\varphi_1} \boldsymbol R_{\varphi_1}^{\textup T} = \boldsymbol R_{\varphi_1} \left( \lambda_1 \boldsymbol{e}_1 \boldsymbol{e}_1^{\textup T} + \mu_* \boldsymbol{e}_2 \boldsymbol{e}_2^{\textup T} \right) \boldsymbol R_{\varphi_1}^{\textup T} \\
        &= \lambda_1 (\boldsymbol R_{\varphi_1} \boldsymbol{e}_1) \left( \boldsymbol R_{\varphi_1} \boldsymbol{e}_1 \right)^{\textup T} + \mu_* (\boldsymbol R_{\varphi_1} \boldsymbol{e}_2) \left( \boldsymbol R_{\varphi_1} \boldsymbol{e}_2 \right)^{\textup T} = \lambda_1 \boldsymbol{u}_1 \boldsymbol{u}_1^{\textup T} + \mu_* \boldsymbol{v}_1 \boldsymbol{v}_1^{\textup T}.
    \end{align*}
    \hfill $\square$
\end{proof}

\noindent Before we look at the case where the largest eigenvalue is not unique, let us calculate a small example with numerical values for better understanding.

\begin{example}
    We consider as RGB colours blue $\boldsymbol C_1 = (0, 0, 1)$, a medium-dark brown $\boldsymbol C_2 = \left(\frac{3}{5}, \frac{2}{5}, \frac{1}{5} \right)$ and a shade of blue-magenta $\boldsymbol C_3 = \left(\frac{1}{3}, \frac{1}{3}, \frac{5}{6}\right)$ and calculate the LES for them. To do this, we first convert the RGB colours into symmetric matrices as described in section 2.1: 
    \begin{align*}
         &\boldsymbol X_1 = \frac{1}{2 \sqrt{2}}
         \begin{pmatrix}
            \sqrt{3} \hspace{0.5em}& -1 \\
            -1 \hspace{0.5em}& -\sqrt{3}
        \end{pmatrix}
        \approx
        \begin{pmatrix}
            \hspace{0.8em}0.6124 \hspace{0.5em}& -0.3535 \\
            -0.3535 \hspace{0.5em}& -0.6124
        \end{pmatrix}
        , \\ 
        &\boldsymbol X_2 = \frac{1}{5 \sqrt{2}}
        \begin{pmatrix}
            -2 \hspace{0.5em}& \sqrt{3} \\
            \sqrt{3} \hspace{0.5em}& 0
        \end{pmatrix}
        \approx
        \begin{pmatrix}
            -0.2828 \hspace{0.5em}& 0.2450 \\
            \hspace{0.8em}0.2450 \hspace{0.5em}& 0
        \end{pmatrix}
        \quad \text{and} \quad \\
        &\boldsymbol X_3 = \frac{1}{12 \sqrt{2}}
        \begin{pmatrix}
            2 + 3 \sqrt{3} \hspace{0.5em}& -3 \\
            -3 \hspace{0.5em}& 2 - 3 \sqrt{3}
        \end{pmatrix}
        \approx
        \begin{pmatrix}
            \hspace{0.8em}0.4240 \hspace{0.5em}& -0.1768 \\
            -0.1768 \hspace{0.5em}& -0.1883
        \end{pmatrix}
        .
    \end{align*}
    Then, we form the spectral decomposition \eqref{spectral_decomposition} with the eigenvalues $\lambda_i, \mu_i$ and eigenvectors $\boldsymbol{u}_i, \boldsymbol{v}_i$ for $i = 1,2,3$:
    \begin{align*}
        &\lambda_1 = \frac{1}{\sqrt{2}} \approx 0.7071,  &&\boldsymbol u_1 = \frac{1}{\sqrt{8 + 4 \sqrt{3}}} 
        \begin{pmatrix}
            -2-\sqrt{3} \\
            1
        \end{pmatrix}
        \approx
        \begin{pmatrix}
            -0.9659 \\
            \hspace{0.8em}0.2588
        \end{pmatrix}
        ,\\ 
        &\mu_1 = - \lambda_1,  &&\boldsymbol v_1 = \frac{1}{\sqrt{8 - 4 \sqrt{3}}}
        \begin{pmatrix}
            2-\sqrt{3} \\
            1
        \end{pmatrix}
        \approx
        \begin{pmatrix}
            -0.2588 \\
            -0.9659
        \end{pmatrix}
        , \\
        & \lambda_2 = \frac{1}{5 \sqrt{2}} \approx 0.1414,  &&\boldsymbol u_2 = \frac{1}{2} 
        \begin{pmatrix}
            1 \\
            \sqrt{3}
        \end{pmatrix}
        \approx
        \begin{pmatrix}
            0.5000 \\
            0.8660
        \end{pmatrix}
        ,\\
        &\mu_2 = - \frac{3}{5 \sqrt{2}} \approx -0.4243,  &&\boldsymbol v_2 = \frac{1}{2}
        \begin{pmatrix}
            -\sqrt{3} \\
            1
        \end{pmatrix}
        \approx
        \begin{pmatrix}
            -0.8660 \\
            \hspace{0.8em}0.5000
        \end{pmatrix}
        \quad \text{and} 
    \end{align*}
    \begin{align*}
        & \lambda_3 = \frac{2}{3 \sqrt{2}} \approx 0.4714,  &&\boldsymbol u_3 = \frac{1}{\sqrt{8 + 4 \sqrt{3}}}
        \begin{pmatrix}
            -2-\sqrt{3} \\
            1
        \end{pmatrix}
        \approx
        \begin{pmatrix}
            -0.9659 \\
            \hspace{0.8em}0.2588
        \end{pmatrix}
        , \\ 
        &\mu_3 = - \frac{1}{3 \sqrt{2}} \approx - 0.2357,  &&\boldsymbol v_3 = \frac{1}{\sqrt{8 - 4 \sqrt{3}}}
        \begin{pmatrix}
            2-\sqrt{3} \\
            1
        \end{pmatrix}
        \approx
        \begin{pmatrix}
            0.2588 \\
            0.9659
        \end{pmatrix}
        .
    \end{align*}
    The largest eigenvalue of this is $\lambda_1$ and it is also unique. According to Theorem 1 we take as eigenvectors for $\boldsymbol S$ the vectors $\boldsymbol u_1$ and $\boldsymbol v_1$. The second largest eigenvalue is $\lambda_3$, but $\boldsymbol u_3 = \boldsymbol u_1$, so we do not consider this eigenvalue. Therefore, the next largest eigenvalue is $\lambda_2$ and since it has a different eigenvector direction than $\lambda_1$, we select it as the second eigenvalue for $\boldsymbol S$. In conclusion, we thus obtain for the LES:
    \begin{align*}
        \boldsymbol S = \lambda_1 \boldsymbol u_1 \boldsymbol u_1^\textup{T} + \lambda_2 \boldsymbol v_1 \boldsymbol v_1^\textup{T} = \frac{1}{5 \sqrt{2}}
        \begin{pmatrix}
            3+\sqrt{3} \hspace{0.5em}& -1 \\
            -1 \hspace{0.5em}& 3-\sqrt{3}
        \end{pmatrix}
        \approx
        \begin{pmatrix}
            \hspace{0.8em}0.6692 \hspace{0.5em}& -0.1414 \\
            -0.1414 \hspace{0.5em}& \hspace{0.8em}0.1793
        \end{pmatrix}
        ,
    \end{align*}
    which represents in the RGB space a medium-light shade of blue-magenta colour: $\left(\frac{3}{5}, \frac{3}{5}, 1\right)$.
\end{example}

\subsection{LES for non-unique largest Eigenvalues}

Until now, we have assumed that the largest eigenvalue $\lambda_1$ is unique, but we now want to show what we get when this is not the case. To do this, we first show that the largest eigenvalue of $\boldsymbol S$ cannot be greater than $\lambda_1$:

\begin{lemma}
    Let $\boldsymbol S$ be defined as in \eqref{LES} and condition \eqref{rotation} be fulfilled according to \eqref{spectral_decomposition}. Furthermore, let $\boldsymbol X_1$ has the largest not necessarily unique eigenvalue $\lambda_1$ of all matrices of $\mathcal{X}$. Then the largest eigenvalue of $\boldsymbol S$ cannot be greater than $\lambda_1$.
\end{lemma}
\begin{proof}
    We have already seen that the largest eigenvalue of $\boldsymbol S$ arises from the Rayleigh product $R(\varepsilon)$, $\abs{\varepsilon} \ll 1 $, which can be estimated as follows:
    \begin{align*}
        R(\varepsilon) &= \frac{1}{1 + \varepsilon^2} \left( \mathrm{e}^{m \lambda_1} + \varepsilon^2 \mathrm{e}^{m \mu_1} + \sum_{i=2}^n \left( (c_i + \varepsilon s_i)^2 \mathrm{e}^{m \lambda_i} + (s_i - \varepsilon c_i)^2 \mathrm{e}^{m \mu_i} \right) \right) \\
        &\leq \frac{1}{1 + \varepsilon^2} \left( \mathrm{e}^{m \lambda_1} + \varepsilon^2 \mathrm{e}^{m \lambda_1} + \sum_{i=2}^n \left( (c_i + \varepsilon s_i)^2 \mathrm{e}^{m \lambda_1} + (s_i - \varepsilon c_i)^2 \mathrm{e}^{m \lambda_1} \right) \right) \\
        &= \frac{\mathrm{e}^{m \lambda_1}}{1 + \varepsilon^2} \left( 1 + \varepsilon^2 + \sum_{i=2}^n \left( (c_i + \varepsilon s_i)^2 + (s_i - \varepsilon c_i)^2 \right) \right) \\
        &= \frac{\mathrm{e}^{m \lambda_1}}{1 + \varepsilon^2} \left( 1 + \varepsilon^2 + \sum_{i=2}^n \left( c_i^2 + 2\varepsilon c_i s_i + \varepsilon^2 s_i^2 + s_i^2 - 2\varepsilon c_i s_i + \varepsilon^2 c_i^2 \right) \right) \\ 
        &= \frac{\mathrm{e}^{m \lambda_1}}{1 + \varepsilon^2} \left( 1 + \varepsilon^2 + \sum_{i=2}^n \left( 1 + \varepsilon^2 \right) \right) = \mathrm{e}^{m \lambda_1} \left( 1 + \sum_{i=2}^n 1 \right) = n \mathrm{e}^{m \lambda_1}.
    \end{align*}
    Thus the largest eigenvalue of $\boldsymbol S$ can be approximated with
    \begin{align*}
        \lim_{m \rightarrow \infty} \frac{1}{m} \log R(\varepsilon) &\leq \lim_{m \rightarrow \infty} \frac{1}{m} \log \left( n \mathrm{e}^{m \lambda_1} \right) = \lim_{m \rightarrow \infty} \frac{1}{m} \left( \log \left( \mathrm{e}^{m \lambda_1} \right) + \log n \right) \\
        &= \lim_{m \rightarrow \infty} \left( \lambda_1 + \frac{\log n}{m} \right) = \lambda_1.
    \end{align*}
    \hfill $\square$
\end{proof}

In a similar way, one can also show that $\boldsymbol S$ is indeed an upper bound for the $\boldsymbol X_i$, $i = 1,\dots,n$. To do this, we prove the following lemma:

\begin{lemma}
    The LES $\boldsymbol S$ according to \eqref{LES} is an upper bound in the Loewner sense for the given matrices $\mathcal{X}$.
\end{lemma}
\begin{proof}
    Apparently one has 
    \begin{align*}
        \sum_{i=1}^n \exp(m \boldsymbol X_i) \geq_\textup L \exp(m \boldsymbol X_j) \quad \forall j \in \{ 1,\dots,n \},  
    \end{align*}
    which in combination with the fact that the logarithm is an operator-monotone function \cite{lowner_monotone}, i.e. $\boldsymbol A \leq_\textup L \boldsymbol B \Longrightarrow \log \boldsymbol A \leq_\textup L \log \boldsymbol B$ for symmetric positive definite matrices $\boldsymbol A,\boldsymbol B$, results in $\boldsymbol S \geq_\textup L \boldsymbol X_j$ for all $j \in \{1,\dots,n\}$. \hfill $\square$
\end{proof}

We now continue our considerations regarding the other largest eigenvalue. If there is a second eigenvalue which is equal to $\lambda_1$, it can then be either $\mu_1$ or one of the $\lambda_j$, since $\lambda_j \geq \mu_j$. If it were $\mu_1$, then $\boldsymbol X_1$ would have the representation $\lambda_1 \boldsymbol I$ and the only matrix with equally aligned eigenvectors that would be greater than or equal to $\lambda_1 \boldsymbol I$ in the Loewner sense without having an even greater eigenvalue would be $\lambda_1 \boldsymbol I$ itself.

For the other case that $\lambda_j = \lambda_1$ for a fixed $j \in \{2,\dots,n\}$ will we assume that without loss of generality $j = 2$ holds. We summarise the argumentation necessary for this in the following theorem:

\begin{theorem}
    Let $\boldsymbol S$ be the LES of the multi-set $\mathcal{X} = \{\boldsymbol X_1,\dots, \boldsymbol X_n\}$, $n \in \mathbb N$, of symmetric real $2 \times 2$ matrices with the spectral decomposition \eqref{spectral_decomposition} and $\lambda_1 = \lambda_2$ the largest eigenvalues of all these matrices.
    Then one has 
    \begin{align*}
        \boldsymbol S 
        =
        \begin{cases}
            \lambda_1 \boldsymbol I, & \text{ if } \boldsymbol u_1 \neq \pm\boldsymbol u_2, \\
            \lambda_1 \boldsymbol u_1 \boldsymbol u_1^\textup{T} + \mu_* \boldsymbol v_1 \boldsymbol v_1^\textup{T}, & \text{ otherwise},
        \end{cases}
    \end{align*}
    where $\mu_* \leq \lambda_1$ is the next largest eigenvalue whose eigenvector $\boldsymbol{v}_*$ is not aligned with $\boldsymbol{u}_1$.
\end{theorem}
\begin{proof}
    1.) To ensure assumption \eqref{rotation}, we begin this proof with the same rotation as in the proof of Theorem 1. Then, we consider some of the properties of $\boldsymbol S$. The first property is that the largest eigenvalue of $\boldsymbol S$ is $\lambda_1$. This is a consequence of Lemma 5 and the fact that the Rayleigh product $R^*(\varepsilon)$ for this case is greater than or equal to the $R(\varepsilon)$ in Lemma 1. Consequently, Lemma 2 is also subject to this relation in accordance with
    \begin{align*}
        \lim_{m \rightarrow \infty} \frac{1}{m} \log R^*(\varepsilon) \geq \lim_{m \rightarrow \infty} \frac{1}{m} \log R(\varepsilon) = \lambda_1.
    \end{align*}
    This leads to the second property, namely the spectral decomposition of the rotated LES $\boldsymbol S_{-\varphi_1}$:
    \begin{align*}
        &\boldsymbol S_{-\varphi_1} = \lambda_1 \bar{\boldsymbol u} \bar{\boldsymbol u}^\textup T + \alpha \bar{\boldsymbol v} \bar{\boldsymbol v}^\textup T, \quad \lambda_1 \geq \alpha \in \mathbb R, \quad \Bar{\boldsymbol u} = (c,s)^\textup T, \quad \Bar{\boldsymbol v} = (-s,c)^\textup T, \\
        &c = \cos(\varphi), \quad s = \sin(\varphi), \quad \varphi \in \left[- \frac{\pi}{2}, \frac{\pi}{2}\right].
    \end{align*}
    Lemma 6 gives us the upper bound property, in particular, this satisfies
    \begin{align*}
        \boldsymbol 0 \leq_\textup L \boldsymbol S_{-\varphi_1} - \boldsymbol X_1 &= \lambda_1 \left(\Bar{\boldsymbol u} \Bar{\boldsymbol u}^\textup T - \boldsymbol u_1 \boldsymbol u_1^\textup T\right) + \alpha \bar{\boldsymbol v} \bar{ \boldsymbol v}^\textup T - \mu_1 \boldsymbol v_1 \boldsymbol v_1^\textup T \\
        &= \lambda_1 
        \begin{pmatrix}
            c^2 - 1 \hspace{0.5em}& cs \\
            cs \hspace{0.5em}& s^2
        \end{pmatrix}
        + \alpha 
        \begin{pmatrix}
            s^2 \hspace{0.5em}& -cs \\
            -cs \hspace{0.5em}& c^2
        \end{pmatrix}
        - 
        \begin{pmatrix}
            0 \hspace{0.5em}& 0 \\
            0 \hspace{0.5em}& \mu_1
        \end{pmatrix}
        \\
        &= 
        \begin{pmatrix}
            \alpha s^2 - \lambda_1 s^2 \hspace{0.5em}& cs (\lambda_1 - \alpha) \\
            cs (\lambda_1 - \alpha) \hspace{0.5em}& \lambda_1 s^2 + \alpha c^2 - \mu_1
        \end{pmatrix}
        \\
        &= 
        \begin{pmatrix}
            s^2(\alpha - \lambda_1) \hspace{0.5em}& cs (\lambda_1 - \alpha) \\
            cs (\lambda_1 - \alpha) \hspace{0.5em}& c^2(\alpha - \lambda_1) + \lambda_1 - \mu_1
        \end{pmatrix},
    \end{align*}
    but this is only fulfilled if
    \begin{align}
        s^2(\alpha - \lambda_1) \geq 0 \quad \wedge \quad c^2(\alpha - \lambda_1) + \lambda_1 - \mu_1 \geq 0 \quad \wedge \quad \det(\boldsymbol S_{-\varphi_1} - \boldsymbol X_1) \geq 0
        \label{not_unique_relations}
    \end{align}
    apply. \vspace{0.2cm}

    \noindent 2.) Let us assume that $s^2 \neq 0$, then it follows from the first relation of \eqref{not_unique_relations} that $\alpha \geq \lambda_1$. Because of Lemma 5 the eigenvalue $\alpha$ cannot be greater than $\lambda_1$, which leads to $\alpha = \lambda_1$. This in turn leads to the spectral decomposition having the following form:
    \begin{align*}
        \boldsymbol S_{-\varphi_1} = \lambda_1 \left( \bar{\boldsymbol u} \bar{\boldsymbol u}^\textup T + \bar{\boldsymbol v} \bar{\boldsymbol v}^\textup T \right) = \lambda_1 
        \begin{pmatrix}
            c^2 + s^2 \hspace{0.5em}& 0 \\
            0 \hspace{0.5em}& c^2 + s^2
        \end{pmatrix}
        = \lambda_1 \boldsymbol I.
    \end{align*}
    By rotating all eigenvectors back by $\boldsymbol R_{\varphi_1}$, we achieve
    \begin{align*}
        \boldsymbol{S} = \boldsymbol R_{\varphi_1} \boldsymbol{S}_{-\varphi_1} \boldsymbol R_{\varphi_1}^{\textup T} = \lambda_1 \boldsymbol R_{\varphi_1} \boldsymbol{I} \boldsymbol R_{\varphi_1}^{\textup T} = \lambda_1 \boldsymbol{I}.
    \end{align*}
    
    \vspace{0.2cm}

    \noindent 3.) Now we assume the other case $s^2 = 0$. Therefore, we have $s = 0$ and $c = \pm 1$. If we insert these values into the spectral decomposition, we obtain
    \begin{align*}
        \boldsymbol S_{-\varphi_1} = \lambda_1
        \begin{pmatrix}
            1 \hspace{0.5em}& 0\\
            0 \hspace{0.5em}& 0
        \end{pmatrix}
        + \alpha 
        \begin{pmatrix}
            0 \hspace{0.5em}& 0 \\
            0 \hspace{0.5em}& 1
        \end{pmatrix}
        =
        \begin{pmatrix}
            \lambda_1 \hspace{0.5em}& 0 \\
            0 \hspace{0.5em}& \alpha
        \end{pmatrix}.
    \end{align*}
    By comparing with equation \eqref{not_unique_relations}, we see that the only non-trivial condition remaining is $\alpha \geq \mu_1$. 
    
    If now $\boldsymbol u_1 = \pm \boldsymbol u_2$ would hold, the matrix $\boldsymbol E_m$ would have the representation
    \begin{align*}
    \begin{pmatrix}
        2\mathrm{e}^{m \lambda_1} + \sum_{i=3}^n \left( c_i^2 \mathrm{e}^{m \lambda_i} + s_i^2 \mathrm{e}^{m \mu_i} \right) \hspace{0.5em}& \sum_{i=3}^n c_i s_i \left( \mathrm{e}^{m\lambda_i} - \mathrm{e}^{m \mu_i} \right) \\
        \sum_{i=3}^n c_i s_i \left( \mathrm{e}^{m\lambda_i} - \mathrm{e}^{m \mu_i} \right) \hspace{0.5em}& \mathrm{e}^{m \mu_1} + \mathrm{e}^{m \mu_2} + \sum_{i=3}^n \left( s_i^2 \mathrm{e}^{m \lambda_i} + c_i^2 \mathrm{e}^{m \mu_i} \right)
    \end{pmatrix}.
    \end{align*}
    Thus we can repeat all the steps up to Theorem 1 with minor adjustments and get the result from the theorem with $\alpha = \mu_*$, where $\mu_*$ is the next largest eigenvalue of $\mathcal{X}$ whose eigenvector $\boldsymbol{v}_*$ is not aligned with $\boldsymbol{u}_1$. By rotating it back, we achieve as in the proof of Theorem 1:
    \begin{align*}
        \boldsymbol{S} &= \boldsymbol R_{\varphi_1} \boldsymbol{S}_{-\varphi_1} \boldsymbol R_{\varphi_1}^{\textup T} = \boldsymbol R_{\varphi_1} \left( \lambda_1 \boldsymbol{e}_1 \boldsymbol{e}_1^{\textup T} + \mu_* \boldsymbol{e}_2 \boldsymbol{e}_2^{\textup T} \right) \boldsymbol R_{\varphi_1}^{\textup T} \\
        &= \lambda_1 \boldsymbol R_{\varphi_1} \boldsymbol{e}_1 \left( \boldsymbol R_{\varphi_1} \boldsymbol{e}_1 \right)^{\textup T} + \mu_* \boldsymbol R_{\varphi_1} \boldsymbol{e}_2 \left( \boldsymbol R_{\varphi_1} \boldsymbol{u}_2 \right)^{\textup T} = \lambda_1 \boldsymbol{u}_1 \boldsymbol{u}_1^{\textup T} + \mu_* \boldsymbol{v}_1 \boldsymbol{v}_1^{\textup T}.
    \end{align*}
    
    Otherwise, we return to the argumentation of Lemma 4 with the difference that $\lambda_2 = \lambda_1$ holds.
    We then obtain for the Rayleigh product in the case $s_2^2 \neq 0$:
    \begin{align*}
        \overline{R}(\varepsilon) &= \frac{\mathrm{e}^{m \lambda_1}}{1 + \varepsilon^2} \bigg( \big( s_2 - \varepsilon c_2 \big)^2 + \varepsilon^2 + \mathrm{e}^{m (\mu_1- \lambda_1)} + \big( c_2 + \varepsilon s_2 \big)^2 \mathrm{e}^{m (\mu_2- \lambda_1)} \\
        & \quad + \sum_{i=3}^n \Big( \big( s_i - \varepsilon c_i \big)^2 \mathrm{e}^{m (\lambda_i - \lambda_1)} + \big( c_i + \varepsilon s_i \big)^2 \mathrm{e}^{m (\mu_i - \lambda_1)} \Big) \bigg) \\
        &= \frac{\mathrm{e}^{m \lambda_1}}{1 + \varepsilon^2} \bigg( \big( s_2 - \varepsilon c_2 \big)^2 + \varepsilon^2 + \mathcal{O}\left( \mathrm{e}^{m (\mu_* - \lambda_1)} \right) \bigg),
    \end{align*}
    where $\mu_*$ is again the next largest eigenvalue of $\mathcal{X}$ whose eigenvector $\boldsymbol{v}_*$ is not aligned with $\boldsymbol{u}_1$. By determining $\frac{1}{m} \log \overline{R}(\varepsilon)$ again for this and then taking the limit for $m \to \infty$, the $\varepsilon$ terms and the $\mathcal{O}(\cdot)$ term will disappear and only $\lambda_1$ remains. Thus, we have $\boldsymbol{S}_{-\varphi_1} = \lambda_1 \boldsymbol{I}$ again. 
    \hfill $\square$
\end{proof}

At this point, we would also like to give a small example in the case that the largest eigenvalue is not unique.

\begin{example}
    Here we consider the simple example of a bipartite image consisting of the two RGB colours blue $\boldsymbol C_1 = (0, 0, 1)$ and green $\boldsymbol C_2 = (0, 1, 0)$. The corresponding symmetric matrices are
    \begin{align*}
        \boldsymbol{X}_1 = \frac{1}{2 \sqrt{2}}
        \begin{pmatrix}
            \sqrt{3} \hspace{0.5em}& -1 \\
            -1 \hspace{0.5em}& -\sqrt{3}
        \end{pmatrix}
        \approx
        \begin{pmatrix}
            \hspace{0.8em} 0.6124 \hspace{0.5em}& -0.3535 \\
            -0.3535 \hspace{0.5em}& -0.6124
        \end{pmatrix}
    \end{align*}
    and 
    \begin{align*}
        \boldsymbol{X}_2 = \frac{1}{2 \sqrt{2}}
        \begin{pmatrix}
            -\sqrt{3} \hspace{0.5em}& -1 \\
            -1 \hspace{0.5em}& \sqrt{3}
        \end{pmatrix}
        \approx
        \begin{pmatrix}
            -0.6124 \hspace{0.5em}& -0.3535 \\
            -0.3535 \hspace{0.5em}& \hspace{0.8em} 0.6124
        \end{pmatrix}.
    \end{align*}
    This results in the following eigenvalues and eigenvectors:
    \begin{align*}
        &\lambda_1 = \frac{1}{\sqrt{2}} \approx 0.7071,   &&\boldsymbol{u}_1 = \frac{1}{2 \sqrt{2+\sqrt{3}}}
        \begin{pmatrix}
            -\sqrt{3} -2 \\
            1
        \end{pmatrix}
        \approx
        \begin{pmatrix}
            -0.9659 \\
            \hspace{0.8em} 0.2588 
        \end{pmatrix}
        , \\ 
        &\mu_1 = - \lambda_1,  &&\boldsymbol{v}_1 = \frac{1}{2 \sqrt{2+\sqrt{3}}}
        \begin{pmatrix}
            -1 \\
            -\sqrt{3} -2
        \end{pmatrix}
        \approx
        \begin{pmatrix}
            -0.2588 \\
            - 0.9659
        \end{pmatrix}
    \end{align*}
    and
    \begin{align*}
        &\lambda_2 = \lambda_1,   &&\boldsymbol{u}_2 = \frac{1}{2 \sqrt{2+\sqrt{3}}}
        \begin{pmatrix}
            -1 \\
            \sqrt{3} +2
        \end{pmatrix}
        \approx
        \begin{pmatrix}
            -0.2588 \\
            \hspace{0.8em} 0.9659
        \end{pmatrix}
        , \\ 
        &\mu_2 = \mu_1,  &&\boldsymbol{v}_2 = \frac{1}{2 \sqrt{2+\sqrt{3}}}
        \begin{pmatrix}
            -\sqrt{3} -2 \\
            -1
        \end{pmatrix}
        \approx
        \begin{pmatrix}
            -0.9659 \\
            -0.2588 
        \end{pmatrix}.
    \end{align*}
    Since the largest eigenvalue is not unique and the corresponding eigenvectors $\boldsymbol{u}_1, \boldsymbol{u}_2$ are not equal, we obtain by Theorem 2 the LES
    \begin{align*}
        \boldsymbol{S} = \lambda_1 \boldsymbol{I} = \frac{1}{\sqrt{2}} \boldsymbol{I} \approx 
        \begin{pmatrix}
            0.7071 \hspace{0.5em}& 0 \\
            0 \hspace{0.5em}& 0.7071
        \end{pmatrix},
    \end{align*}
    which represents the RGB colour white $(1, 1, 1)$.
\end{example}

\subsection{General Characterisation and Properties of the LES}

In this section we will give a general characterisation of the LES for any combination of eigenvalues. Based on this, we will show two interesting properties in the form of transitivity and associativity with respect to dilation. The latter property in particular sets our approach apart from other multidimensional approaches, since, to our knowledge, it does not exist in colour morphology as opposed to grey-scale morphology.

In order to provide greater clarity, we shall now present a summary of the two preceding theorems in the form of a corollary that gives us a general characterisation of the LES.

\begin{corollary}
    Let $\boldsymbol{S}$ be the LES \eqref{LES} of the multi-set $\mathcal{X} = \{\boldsymbol{X}_1, \dots, \boldsymbol{X}_n\}$, $n \in \mathbb{N}$, of symmetric real $2 \times 2$ matrices with the spectral decompositions \eqref{spectral_decomposition} and $\lambda_1$ (one of) the largest eigenvalues of all these matrices. Furthermore, let $\mathcal{V}(\mathcal{X})$ be the set of all eigenvectors of the matrices of $\mathcal{X}$ and $\mathcal{V}_{\sup}^{\lambda_1}(\mathcal{X})$ be the set of corresponding eigenvectors to the largest eigenvalues equal to $\lambda_1$, i.e.
    \begin{align*}
        \mathcal{V}_{\sup}^{\lambda_1}(\mathcal{X}) := \{ \boldsymbol{v} \in \mathcal{V}(\mathcal{X}) : \exists \boldsymbol{X} \in \mathcal{X}: \boldsymbol{X} \boldsymbol{v} = \lambda \boldsymbol{v} \wedge \lambda = \lambda_1 \}.
    \end{align*}
    Then, the LES can be characterised as follows:
    \begin{align}
        \boldsymbol{S} = 
        \begin{cases}
            \lambda_1 \boldsymbol{I}, & \textup{if } \lambda_1 \textup{ is not unique and } \exists \boldsymbol{v} \in \mathcal{V}_{\sup}^{\lambda_1}(\mathcal{X}): \boldsymbol{v} \neq \pm \boldsymbol{u}_1, \\
            \lambda_1 \boldsymbol{u}_1 \boldsymbol{u}_1^{\textup T} + \mu_* \boldsymbol{v}_1 \boldsymbol{v}_1^{\textup T}, & \textup{otherwise},
        \end{cases}
        \label{LES_computed}
    \end{align}
    where $\mu_* \leq \lambda_1$ is the next largest eigenvalue of $\mathcal{X}$ whose corresponding eigenvector $\boldsymbol{v}_*$ is not aligned with $\boldsymbol{u}_1$.
\end{corollary}

We will show in the following proposition that the LES $\boldsymbol S$, as previously characterised by \eqref{LES_computed}, exhibits a transitive property in general.

\begin{proposition}
    The LES \eqref{LES} is transitive, i.e. for multi-sets $\mathcal{X}$ and $\mathcal{Y}$ of symmetric $2 \times 2$ matrices one has 
    \begin{align}
        \Sup_{\textup{LE}}(\mathcal{X} \cup \mathcal{Y}) = \Sup_{\textup{LE}}\big(\{\Sup_{\textup{LE}}(\mathcal{X}), \Sup_{\textup{LE}}(\mathcal{Y})\}\big).
        \label{Transitivity}
    \end{align}
\end{proposition}
\begin{proof}
    1.) For the multi-sets we will use the notation
    \begin{align*}
        &\mathcal{X} = \{\boldsymbol X_1,\dots , \boldsymbol X_n\}, \quad \mathcal{Y} = \{\boldsymbol Y_1,\dots,\boldsymbol Y_m\}, \quad m,n \in \mathbb N, \\
        &\mathcal{Z} = \{\boldsymbol Z_1,\dots,\boldsymbol Z_{n+m}\} := \mathcal{X} \cup \mathcal{Y} = \{\boldsymbol X_1,\dots,\boldsymbol X_n,\boldsymbol Y_1,\dots,\boldsymbol Y_m\}.
    \end{align*}
    The right-hand side of \eqref{Transitivity} can be calculated as
    \begin{align*}
        &\boldsymbol{S}_r := \Sup_{\textup{LE}}\big(\{\boldsymbol{S}_1, \boldsymbol{S}_2\}\big) := \Sup_{\textup{LE}}\big(\{\Sup_{\textup{LE}}(\mathcal{X}), \Sup_{\textup{LE}}(\mathcal{Y})\}\big) 
        .
    \end{align*}
    The LES $\boldsymbol{S}_1$ and $\boldsymbol{S}_2$ are characterised according to \eqref{LES_computed} as
    \begin{align*}
        \boldsymbol{S}_1 = 
        \begin{cases}
            \lambda_1 \boldsymbol{I}, & \text{if } \lambda_1 \text{ is not unique and } \exists \boldsymbol{v} \in \mathcal{V}_{\sup}^{\lambda_1}(\mathcal{X}): \boldsymbol{v} \neq \pm \boldsymbol{u}_1, \\
            \lambda_1 \boldsymbol{u}_1 \boldsymbol{u}_1^{\textup T} + \lambda_* \boldsymbol{v}_1 \boldsymbol{v}_1^{\textup T}, & \text{otherwise},
        \end{cases}
    \end{align*}
    and
    \begin{align*}
        \boldsymbol{S}_2 = 
        \begin{cases}
            \mu_1 \boldsymbol{I}, & \text{if } \mu_1 \text{ is not unique and } \exists \boldsymbol{v} \in \mathcal{V}_{\sup}^{\mu_1}(\mathcal{Y}): \boldsymbol{v} \neq \pm \tilde{\boldsymbol{u}}_1, \\
            \mu_1 \tilde{\boldsymbol{u}}_1 \tilde{\boldsymbol{u}}_1^{\textup T} + \mu_* \tilde{\boldsymbol{v}}_1 \tilde{\boldsymbol{v}}_1^{\textup T}, & \text{otherwise}.
        \end{cases}
    \end{align*}
    For the characterisation of $\boldsymbol{S}_1$ we assumed that $\lambda_1$ is the largest eigenvalue of $\mathcal{X}$ with the corresponding eigenvector $\boldsymbol{u}_1$, $\lambda_*$ is the next largest eigenvalue of $\mathcal{X}$ whose eigenvector is not aligned with $\boldsymbol{u}_1$ and $\boldsymbol{v}_1$ is the eigenvector perpendicular to $\boldsymbol{u}_1$. The same can be said of the characterisation of $\boldsymbol{S}_2$, insofar as we use $\mathcal{Y}$, $\mu_1$ and $\mu_*$ instead of $\lambda_1$ and $\lambda_*$ and the eigenvectors $\tilde{\boldsymbol{u}}_1$ and $\tilde{\boldsymbol{v}}_1$ as replacement for $\boldsymbol{u}_1$ and $\boldsymbol{v}_1$. 
    
    Since $\boldsymbol{S}_r$ is again an LES approximation for two symmetric $2 \times 2$ matrices, we can characterise it as follows
    \begin{align*}
        \boldsymbol{S}_r = 
        \begin{cases}
            \eta_1 \boldsymbol{I} &, \text{if } \eta_1 \text{ is not unique and } \exists \boldsymbol{v} \in \mathcal{V}_{\sup}^{\eta_1}(\mathcal{S}): \boldsymbol{v} \neq \pm \bar{\boldsymbol{u}}_1, \\
            \eta_1 \bar{\boldsymbol{u}}_1 \bar{\boldsymbol{u}}_1^{\textup T} + \eta_* \bar{\boldsymbol{v}}_1 \bar{\boldsymbol{v}}_1^{\textup T} &, \text{otherwise},
        \end{cases}
    \end{align*}
    where $\eta_1$ represents the maximum of $\lambda_1$ and $\mu_1$ and it is associated with the corresponding normalised eigenvector $\bar{\boldsymbol u}_1$. The term $\eta_*$, on the other hand, denotes the next largest eigenvalue within the set $\{ \lambda_1, \lambda_*, \mu_1, \mu_* \}$ of possible eigenvalues and, in addition, corresponds to an eigenvector that is not aligned with $\bar{\boldsymbol u}_1$. The eigenvector $\bar{\boldsymbol{v}}_1$ denotes the eigenvector perpendicular to $\bar{\boldsymbol u}_1$ and $\mathcal{S}$ is defined as $\mathcal{S} := \{ \boldsymbol{S}_1, \boldsymbol{S}_2 \}$.

    \vspace{0.2cm}
    
    \noindent 2.) For the left-hand side of \eqref{Transitivity}, we obtain with the characterisation \eqref{LES_computed}:
    \begin{align*}
        \boldsymbol{S}_l &:= \Sup_{\textup{LE}}(\mathcal{Z}) \\
        &\,= 
        \begin{cases}
            \nu_1 \boldsymbol{I}, & \text{if } \nu_1 \text{ is not unique and } \exists \boldsymbol{v} \in \mathcal{V}_{\sup}^{\nu_1}(\mathcal{Z}): \boldsymbol{v} \neq \pm \hat{\boldsymbol{u}}_1, \\
            \nu_1 \hat{\boldsymbol{u}}_1 \hat{\boldsymbol{u}}_1^{\textup T} + \nu_* \hat{\boldsymbol{v}}_1 \hat{\boldsymbol{v}}_1^{\textup T}, & \text{otherwise},
        \end{cases}
    \end{align*}
    where $\nu_1$ is (one of) the largest eigenvalues of $\mathcal{Z}$ with the corresponding normalised eigenvector $\hat{\boldsymbol u}_1$, $\nu_*$ is the next largest eigenvalue of $\mathcal{Z}$ with an eigenvector that is not aligned with $\nu_1$ and $\hat{\boldsymbol v}_1$ is the normalised eigenvector perpendicular to $\hat{\boldsymbol u}_1$. Since $\nu_1$ is the largest eigenvalue of $\mathcal{Z} = \mathcal{X} \cup \mathcal{Y}$, it is also the largest eigenvalue of $\mathcal{X}$ and $\mathcal{Y}$ and as such it fulfils $\nu_1 = \eta_1$. 

    \vspace{0.2cm}
    
    \noindent 3.) The first case of $\boldsymbol{S}_l$ will only happen if there are at least two largest eigenvalues in $\mathcal{Z}$ whose eigenvectors are not aligned. As they would be the largest eigenvalues of $\mathcal{Z}$, they would also be the largest eigenvalues of $\mathcal{X}$ or $\mathcal{Y}$. In the event that both values correspond to matrices from the multi-set $\mathcal{X}$, then $\boldsymbol S_1$ would be a diagonal matrix with these eigenvalues. The same would apply for $\boldsymbol{S}_2$ if they would both correspond to matrices from $\mathcal{Y}$. This implies that in these cases, one of $\boldsymbol{S}_1$ and $\boldsymbol{S}_2$ would be a diagonal matrix with these eigenvalues. As they are part of the largest eigenvalues, the corresponding matrix would be selected for $\boldsymbol{S}_r$. 
    
    For the case that these two eigenvalues are distributed between $\mathcal{X}$ and $\mathcal{Y}$, they would nevertheless be selected as either $\lambda_1$, $\lambda_*$, $\mu_1$ or $\mu_*$. This is because they remain the largest eigenvalues, and any other largest eigenvalues would have the same eigenvectors as these two, which would then be deemed equivalent to the corresponding one of the two largest eigenvalues with non-aligned eigenvectors. Consequently, the corresponding eigenvectors would be in $\mathcal{V}_{\sup}^{\eta_1}(\mathcal{S})$, and $\boldsymbol{S}_r$ would take the form of a diagonal matrix with these two eigenvalues. In conclusion, The first case of $\boldsymbol{S}_l$ occurs if and only if it occurs at $\boldsymbol{S}_r$ and they coincide. 

    \vspace{0.2cm}
    
    \noindent 4.) For the second case, we use that we have already identified $\eta_1$ with $\nu_1$. This implies that the corresponding eigenvectors are also the same, that is, we have $\hat{\boldsymbol u}_1 = \bar{\boldsymbol u}_1$ and $\hat{\boldsymbol v}_1 = \bar{\boldsymbol v}_1$. Ultimately, it follows from the construction of the the corresponding characterisations shown above and a similar reasoning as in the first case that $\nu_2 = \eta_2$ must hold. Furthermore, it can be shown that $\boldsymbol{S}_l$ and $\boldsymbol{S}_r$ are identical in this case. \hfill $\square$
\end{proof}

To conclude this section, we want to use this property to prove the associativity of LES-dilation. For this purpose, we first define what we understand by LES-dilation and LES-erosion. 

\begin{definition}
    Let $\Omega \subseteq \mathbb{Z}^2$ be the two-dimensional, discrete image domain, $\boldsymbol{f}: \Omega \to \mathbb R^3$ a colour image and $\boldsymbol{b}: \mathbb{Z}^2 \to \mathbb R^3 \cup \{ (-\infty, -\infty, -\infty)^{\textup T} \}$ the structuring function with components according to \eqref{SE} of the set $B_0 \subset \mathbb Z^2$. Furthermore, let $\boldsymbol{\tau}: \mathbb R^3 \to \textup{Sym}(2)$ be the transformation from the vector of the corresponding colour space to the symmetric $2 \times 2$ matrix according to subsection 2.1. Then we define the \textbf{LES-dilation} as
    \begin{align}
    \begin{split}
        &\left( \boldsymbol{f} \oplus_{\textup{LES}} \boldsymbol{b} \right) (\boldsymbol{x}) := \boldsymbol{\tau}^{-1} \Big( \Sup_{\textup{LE}} \big( \big\{ \boldsymbol{\tau}(\boldsymbol{f}(\boldsymbol{x} - \boldsymbol{u})) + \boldsymbol{\tau}(\boldsymbol{b}(\boldsymbol{u})) : \boldsymbol{u} \in \mathbb Z^2 \big\} \big) \Big), 
        \quad \boldsymbol{x} \in \Omega.
    \end{split} \label{LES-dil}
    \end{align}
    In accordance with the duality \eqref{duality_dil-ero} between dilation and erosion for the grey value case, we can also define the \textbf{LES-erosion} here by means of LES-dilation as follows:
    \begin{align}
    \begin{split}
        \left( \boldsymbol{f} \ominus_{\textup{LES}} \boldsymbol{b} \right)(\boldsymbol{x}) := \left( \boldsymbol{f}^c \oplus_{\textup{LES}} \boldsymbol{b} \right)^c (\boldsymbol{x}), \quad \boldsymbol{x} \in \Omega,
    \end{split} \label{LES-ero}
    \end{align}
    where the complement of a colour vector is to be understood as an component-wise complementation according to \eqref{comp_image}:
    \begin{align*}
        \boldsymbol{f}^c(\boldsymbol{x}) &= \left( (R, G, B)^{\textup T} \right)^c := (R^c, G^c, B^c)^{\textup T} \\
        &= (R_{\max} - R + R_{\min}, G_{\max} - G + G_{\min}, B_{\max} - B + B_{\min})^{\textup T}, \quad \boldsymbol{x} \in \Omega.
    \end{align*}
    Here, $R_{\max}$ represents the largest red value of the image $\boldsymbol{f}$ and $R_{\min}$ the smallest; the same applies to the corresponding $G$ and $B$ terms.
\end{definition}

\noindent We now turn to the mentioned associativity of LES-dilation.

\begin{theorem}
    The LES-dilation \eqref{LES-dil} is associative, i.e. for a colour image $\boldsymbol f: \mathbb{Z}^2 \rightarrow [0,1]^3$ and for the structuring elements given by the structuring functions $\boldsymbol b_1, \boldsymbol b_2: \mathbb{Z}^2 \rightarrow \mathbb{R}^3 \cup \left\{ (-\infty, -\infty, -\infty)^\textup T \right\}$ of two sets $B_1, B_2 \subset \mathbb Z^2$ applies 
    \begin{align}
        (\boldsymbol f \oplus_{\textup{LES}} \boldsymbol b_1) \oplus_{\textup{LES}} \boldsymbol b_2 = \boldsymbol f \oplus_{\textup{LES}} (\boldsymbol b_1 \oplus_{\textup{LES}} \boldsymbol b_2).
        \label{associativity}
    \end{align}
\end{theorem}
\begin{proof}
    To proof this claim, we will compare the considered matrices for both sides of the equation. We use the expression $\boldsymbol{X}(\boldsymbol{x}) := \boldsymbol{\tau}(\boldsymbol{f}(\boldsymbol{x}))$ to notate the matrices of the image $\boldsymbol{f}$. Similarly, we do this for the structuring functions of $B_1$ and $B_2$ using the matrices $\boldsymbol{W}_1(\boldsymbol{x}) := \boldsymbol{\tau}(\boldsymbol{b}_1(\boldsymbol{x}))$ and $\boldsymbol{W}_2(\boldsymbol{x}) := \boldsymbol{\tau}(\boldsymbol{b}_2(\boldsymbol{x}))$.

    For the left-hand side of equation \eqref{associativity}, we calculate for the first dilation
    \begin{align*}
        (\boldsymbol{f} \oplus_{\textup{LES}} \boldsymbol{b}_1)(\boldsymbol{x}) = \boldsymbol{\tau}^{-1} \Big( \Sup_{\textup{LE}} \big(\{\boldsymbol{X}(\boldsymbol{x} - \boldsymbol{y}) + \boldsymbol{W}_1(\boldsymbol{y}): \boldsymbol{y} \in B_1\}\big) \Big), \quad \boldsymbol{x} \in \Omega,
    \end{align*}
    and by the second dilation
    \begin{align*}
        &\big((\boldsymbol f \oplus_{\textup{LES}} \boldsymbol b_1) \oplus_{\textup{LES}} \boldsymbol b_2\big)(\boldsymbol{x}) \\
        &= \boldsymbol{\tau}^{-1} \bigg( \Sup_{\textup{LE}} \Big( \big\{ \boldsymbol{\tau} \big( (\boldsymbol{f} \oplus_{\textup{LES}} \boldsymbol{b}_1)(\boldsymbol{x} - \boldsymbol{z}) \big) + \boldsymbol{W}_2(\boldsymbol{z}): \boldsymbol{z} \in B_2 \big\} \Big) \bigg) \\
        &= \boldsymbol{\tau}^{-1} \bigg( \Sup_{\textup{LE}} \Big( \big\{ \Sup_{\textup{LE}} \big(\{\boldsymbol{X}(\boldsymbol{x} - \boldsymbol{y} - \boldsymbol{z}) + \boldsymbol{W}_1(\boldsymbol{y}): \boldsymbol{y} \in B_1\}\big)\\
        & \quad + \boldsymbol{W}_2(\boldsymbol{z}): \boldsymbol{z} \in B_2 \big\} \Big) \bigg), \quad \boldsymbol{x} \in \Omega. 
    \end{align*}
    Since $B_2$ is a discrete set, we can numerate the elements of $B_2$ with $\boldsymbol{z}_1, \dots, \boldsymbol{z}_m$, $m \in \mathbb N$, and rewrite the equation as
    \begin{align*}
        &\big((\boldsymbol f \oplus_{\textup{LES}} \boldsymbol b_1) \oplus_{\textup{LES}} \boldsymbol b_2\big)(\boldsymbol{x}) \\
        &= \boldsymbol{\tau}^{-1} \bigg( \Sup_{\textup{LE}} \Big( \Sup_{\textup{LE}} \big(\{\boldsymbol{X}(\boldsymbol{x} - \boldsymbol{y} - \boldsymbol{z}_1) + \boldsymbol{W}_1(\boldsymbol{y}): \boldsymbol{y} \in B_1\}\big) + \boldsymbol{W}_2(\boldsymbol{z}_1), \dots, \\
        & \quad \Sup_{\textup{LE}} \big(\{\boldsymbol{X}(\boldsymbol{x} - \boldsymbol{y} - \boldsymbol{z}_m) + \boldsymbol{W}_1(\boldsymbol{y}): \boldsymbol{y} \in B_1\}\big) + \boldsymbol{W}_2(\boldsymbol{z}_m) \Big) \bigg) 
    \end{align*}
    \begin{align*}
        &= \boldsymbol{\tau}^{-1} \bigg( \Sup_{\textup{LE}} \Big( \Sup_{\textup{LE}} \big(\{\boldsymbol{X}(\boldsymbol{x} - \boldsymbol{y} - \boldsymbol{z}_1) + \boldsymbol{W}_1(\boldsymbol{y}) + \boldsymbol{W}_2(\boldsymbol{z}_1): \boldsymbol{y} \in B_1\}\big) , \dots, \\
        & \quad \Sup_{\textup{LE}} \big(\{\boldsymbol{X}(\boldsymbol{x} - \boldsymbol{y} - \boldsymbol{z}_m) + \boldsymbol{W}_1(\boldsymbol{y}) + \boldsymbol{W}_2(\boldsymbol{z}_m): \boldsymbol{y} \in B_1\}\big) \Big) \bigg), \quad \boldsymbol{x} \in \Omega,
    \end{align*}
    since the summation with a constant matrix $\boldsymbol{W}_2(\boldsymbol{z}_i)$ does not influence the decision of the LES except of a global (for each individual LES) rotation of the eigenvectors and change of the eigenvalues. Then by setting 
    \begin{align*}
        \mathcal{Y}_i(\boldsymbol{x}) := \{\boldsymbol{X}(\boldsymbol{x} - \boldsymbol{y} - \boldsymbol{z}_i) + \boldsymbol{W}_1(\boldsymbol{y}) + \boldsymbol{W}_2(\boldsymbol{z}_i): \boldsymbol{y} \in B_1\}, \quad \boldsymbol{x} \in \Omega,
    \end{align*}
    and using the transitivity \eqref{Transitivity}, we obtain for all $\boldsymbol{x} \in \Omega$:
    \begin{align}
        &\big((\boldsymbol f \oplus_{\textup{LES}} \boldsymbol b_1) \oplus_{\textup{LES}} \boldsymbol b_2\big)(\boldsymbol{x}) \nonumber\\
        &= \boldsymbol{\tau}^{-1} \bigg( \Sup_{\textup{LE}} \Big( \Sup_{\textup{LE}} \big( \mathcal{Y}_1(\boldsymbol{x}) \big), \dots, \Sup_{\textup{LE}} \big( \mathcal{Y}_m(\boldsymbol{x}) \big) \Big) \bigg) \nonumber\\
        &= \boldsymbol{\tau}^{-1} \bigg( \Sup_{\textup{LE}} \Big( \bigcup_{i = 1}^m \mathcal{Y}_i(\boldsymbol{x}) \Big) \bigg) \nonumber\\
        &= \boldsymbol{\tau}^{-1} \bigg( \Sup_{\textup{LE}} \Big( \big\{ \boldsymbol{X}(\boldsymbol{x} - \boldsymbol{y} - \boldsymbol{z}) + \boldsymbol{W}_1(\boldsymbol{y}) + \boldsymbol{W}_2(\boldsymbol{z}): \boldsymbol{y} \in B_1 \wedge \boldsymbol{z} \in B_2 \big\} \Big) \bigg). \label{trans-left}
    \end{align}

    The right-hand side of \eqref{associativity} fulfils for the first dilation with the structuring functions the equality
    \begin{align*}
        (\boldsymbol b_1 \oplus_{\textup{LES}} \boldsymbol b_2)(\boldsymbol{y}) = \boldsymbol{\tau}^{-1} \Big( \Sup_{\textup{LE}} \big(\{\boldsymbol{W}_1(\boldsymbol{y} - \boldsymbol{z}) + \boldsymbol{W}_2(\boldsymbol{z}): \boldsymbol{z} \in B_2\}\big) \Big), \quad \boldsymbol{y} \in \mathbb Z^2.
    \end{align*}
    However, since $\boldsymbol{b}_1$ and as such $\boldsymbol{W}_1$ will ``vanish" with $-\infty$ in the sense of the dilation for $\boldsymbol{y} - \boldsymbol{z} \notin B_1$, we can introduce the set 
    \begin{align*}
        B_1 \oplus B_2 := \{ \boldsymbol{x}_1 + \boldsymbol{x}_2 : \boldsymbol{x}_1 \in B_1 \wedge \boldsymbol{x}_2 \in B_2 \}
    \end{align*}
    and replace $\boldsymbol{y} \in Z^2$ with $\boldsymbol{y} \in B_1 \oplus B_2$. We calculate the second dilation as follows
    \begin{align*}
        &\big( \boldsymbol f \oplus_{\textup{LES}} (\boldsymbol b_1 \oplus_{\textup{LES}} \boldsymbol b_2) \big) (\boldsymbol{x}) \\
        &= \boldsymbol{\tau}^{-1} \bigg( \Sup_{\textup{LE}} \Big( \big\{ \boldsymbol{X}(\boldsymbol{x} - \boldsymbol{y}) \\
        &\quad + \Sup_{\textup{LE}} \big(\{\boldsymbol{W}_1(\boldsymbol{y} - \boldsymbol{z}) + \boldsymbol{W}_2(\boldsymbol{z}): \boldsymbol{z} \in B_2\}\big) : \boldsymbol{y} \in B_1 \oplus B_2 \big\} \Big) \bigg), \quad \boldsymbol{x} \in \Omega.
    \end{align*}
    Given that $B_1 \oplus B_2$ is also a discrete set, the same countability trick that was employed previously can be applied, resulting with $\boldsymbol{y}_1, \dots, \boldsymbol{y}_k$ in the following:
    \begin{align*}
        &\big( \boldsymbol f \oplus_{\textup{LES}} (\boldsymbol b_1 \oplus_{\textup{LES}} \boldsymbol b_2) \big) (\boldsymbol{x}) \\
        &= \boldsymbol{\tau}^{-1} \bigg( \Sup_{\textup{LE}} \Big(  \boldsymbol{X}(\boldsymbol{x} - \boldsymbol{y}_1) + \Sup_{\textup{LE}} \big(\{\boldsymbol{W}_1(\boldsymbol{y}_1 - \boldsymbol{z}) + \boldsymbol{W}_2(\boldsymbol{z}): \boldsymbol{z} \in B_2\}\big), \dots, \\
        &\quad \boldsymbol{X}(\boldsymbol{x} - \boldsymbol{y}_k) + \Sup_{\textup{LE}} \big(\{\boldsymbol{W}_1(\boldsymbol{y}_k - \boldsymbol{z}) + \boldsymbol{W}_2(\boldsymbol{z}): \boldsymbol{z} \in B_2\}\big) \Big) \bigg) 
    \end{align*}
    \begin{align*}
        &= \boldsymbol{\tau}^{-1} \bigg( \Sup_{\textup{LE}} \Big( \Sup_{\textup{LE}} \big(\{ \boldsymbol{X}(\boldsymbol{x} - \boldsymbol{y}_1) + \boldsymbol{W}_1(\boldsymbol{y}_1 - \boldsymbol{z}) + \boldsymbol{W}_2(\boldsymbol{z}): \boldsymbol{z} \in B_2\}\big), \dots, \\
        &\quad  \Sup_{\textup{LE}} \big(\{ \boldsymbol{X}(\boldsymbol{x} - \boldsymbol{y}_k) + \boldsymbol{W}_1(\boldsymbol{y}_k - \boldsymbol{z}) + \boldsymbol{W}_2(\boldsymbol{z}): \boldsymbol{z} \in B_2\}\big) \Big) \bigg), \quad \boldsymbol{x} \in \Omega.
    \end{align*}
    By defining 
    \begin{align*}
        \mathcal{Z}_i(\boldsymbol{x}) := \{ \boldsymbol{X}(\boldsymbol{x} - \boldsymbol{y}_i) + \boldsymbol{W}_1(\boldsymbol{y}_i - \boldsymbol{z}) + \boldsymbol{W}_2(\boldsymbol{z}): \boldsymbol{z} \in B_2\}, \quad x \in \Omega,
    \end{align*}
    and using the transitivity \eqref{Transitivity}, we obtain for all $\boldsymbol{x} \in \Omega$:
    \begin{align*}
        &\big( \boldsymbol f \oplus_{\textup{LES}} (\boldsymbol b_1 \oplus_{\textup{LES}} \boldsymbol b_2) \big) (\boldsymbol{x}) \\
        &= \boldsymbol{\tau}^{-1} \bigg( \Sup_{\textup{LE}} \Big( \Sup_{\textup{LE}} \big( \mathcal{Z}_1(\boldsymbol{x}) \big), \dots, \Sup_{\textup{LE}} \big( \mathcal{Z}_k(\boldsymbol{x}) \big) \Big) \bigg) \\
        &= \boldsymbol{\tau}^{-1} \bigg( \Sup_{\textup{LE}} \Big( \bigcup_{i = 1}^k \mathcal{Z}_i(\boldsymbol{x}) \Big) \bigg) \\
        &= \boldsymbol{\tau}^{-1} \bigg( \Sup_{\textup{LE}} \Big( \bigcup_{\boldsymbol{y} \in B_1 \oplus B_2} \{ \boldsymbol{X}(\boldsymbol{x} - \boldsymbol{y}) + \boldsymbol{W}_1(\boldsymbol{y} - \boldsymbol{z}) + \boldsymbol{W}_2(z) : \boldsymbol{z} \in B_2 \} \Big) \bigg).
    \end{align*}
    By substituting $\boldsymbol{y} = \hat{\boldsymbol{y}} + \boldsymbol{z}$ with $\hat{\boldsymbol{y}} \in B_1$ and $\boldsymbol{z} \in B_2$ in the last equation, we see that it equals \eqref{trans-left}, which concludes the proof.
    \hfill $\square$
\end{proof}

\begin{remark}
    Analogously, it can be shown for LES-erosion \eqref{LES-ero} that
    \begin{align*}
        (\boldsymbol{f} \ominus_{\textup{LES}} \boldsymbol{b}_1) \ominus_{\textup{LES}} \boldsymbol{b}_2  = \boldsymbol{f} \ominus_{\textup{LES}} (\boldsymbol{b}_1 \oplus_{\textup{LES}} \boldsymbol{b}_2)
    \end{align*}
    applies.
\end{remark}

\section{Minimality of the LES}

So far we have only shown that the LES \eqref{LES} is an upper bound in the sense of the Loewner order. In this section we address the question of whether it is also a smallest upper bound or under which conditions it is. To do this, we start from the framework \cite{WelkQuantile} of Welk, Kleefeld and Breu{\ss} and add another function to our semi-order to select a smallest upper bound from the set of upper bounds introduced by the semi-order, as already mentioned in section 2.2. This auxiliary function serves to find a unique minimum within the given set, so that we may have no total order but an auxiliary order. This is the reason why we refer to it as the auxiliary ordering function.

The main problem here is the existence of a total ordering function $\varphi$ on the set $\mathcal{U}(\mathcal{X})$ of upper bounds in the Loewner sense of the symmetric matrices $\mathcal{X} = \{ \boldsymbol X_1, \dots, \boldsymbol X_n \}$, $n \in \mathbb N$, for which the LES $\boldsymbol S$ is the unique minimiser. To see this, we first prove the following lemma. 

\begin{lemma}
    Let $\boldsymbol S$ be the LES \eqref{LES} for a given multi-set $\mathcal{X}$ of symmetric $2 \times 2$ matrices and let $\lambda_1$ be the unique largest eigenvalue of $\mathcal{X}$. Whenever the second largest eigenvalue $\lambda_2$ of $\mathcal{X}$ is unique, and its original eigenvector $\boldsymbol u_2$ is not perpendicular to the eigenvector $\boldsymbol u_1$ of $\lambda_1$, there exists some matrix 
    \begin{align}
        \boldsymbol S' := \lambda_1 \boldsymbol u_1 \boldsymbol u_1^\textup T + (\lambda_2 - \varepsilon) \boldsymbol v_1 \boldsymbol v_1^\textup T \in \mathcal{U}(\mathcal{X}), \quad \varepsilon > 0.
        \label{smaller_LES}
    \end{align}
\end{lemma}
\begin{proof}
    Due to Theorem 1, $\boldsymbol S = \lambda_1 \boldsymbol u_1 \boldsymbol u_1^\textup{T} + \lambda_2 \boldsymbol v_1 \boldsymbol v_1^\textup{T}$ holds, where $\boldsymbol v_1$ is the eigenvector perpendicular to $\boldsymbol u_1$. Since we can add a constant value to all eigenvalues and subtract it again later when we have formed the LES, we assume without loss of generality that $\lambda_1 = 1$, and its associated eigenvector is $\boldsymbol u_1 = (1,0)^\textup T$. This makes $\boldsymbol v_1 = (0,1)^\textup T$ and we get the representation
    \begin{align*}
        \boldsymbol S = 
        \begin{pmatrix}
            1 \hspace{0.5em}& 0 \\
            0 \hspace{0.5em}& \lambda_2
        \end{pmatrix}
        , \quad \lambda_2 < 1.
    \end{align*}
    Let 
    \begin{align*}
        \boldsymbol S' = 
        \begin{pmatrix}
            1 \hspace{0.5em}& 0 \\
            0 \hspace{0.5em}& q
        \end{pmatrix}
        , \quad q < \lambda_2.
    \end{align*}
    Then we examine for which $q$ the matrix $\boldsymbol S'$ remains in $\mathcal{U}(\mathcal{X})$. 
    
    For this we choose an arbitrary $\boldsymbol X \in \mathcal{X}$ with a major eigenvalue $\lambda$ satisfying $\lambda \leq \lambda_2$ and the associated eigenvector $\boldsymbol u = (c,s)^\textup T$ with $c = \cos(\varphi)$ and $s = \sin(\varphi)$. Further, let $\mu$ be the minor eigenvalue of $\boldsymbol X$ and have the associated eigenvector $\boldsymbol v = (-s,c)^\textup T$. Then $\boldsymbol X$ has the following representation 
    \begin{align*}
        \boldsymbol X = 
        \begin{pmatrix}
            \lambda c^2 + \mu s^2 \hspace{0.5em}& cs(\lambda - \mu) \\
            cs(\lambda - \mu) \hspace{0.5em}& \lambda s^2 + \mu c^2
        \end{pmatrix}
    \end{align*}
    because of its spectral decomposition \eqref{spectral_decomposition}. Note that $ \boldsymbol{Y} := \boldsymbol S' - \boldsymbol X \geq_\textup L \boldsymbol 0$ must necessarily be satisfied for $\boldsymbol S' \in \mathcal{U}(\mathcal{X})$ to hold.  
    To verify $\boldsymbol{Y} \geq_\textup L \boldsymbol 0$, the following relations must apply:
    \begin{align*}
        R_{\boldsymbol Y}(\boldsymbol u_1) \geq 0 \quad \wedge \quad R_{\boldsymbol Y}(\boldsymbol u_2) \geq 0 \quad \wedge \quad \det(\boldsymbol{Y}) \geq 0.
    \end{align*}
    We calculate for the first term
    \begin{align}
        R_{\boldsymbol Y}(\boldsymbol u_1) = 1 - \lambda c^2 - \mu s^2 \geq 1 - \lambda c^2 - \lambda s^2 = 1 - \lambda \geq 1 - \lambda_2 > 0. \label{ineq_pos_EW}
    \end{align}
    In order for the second term, $R_{\boldsymbol Y}(\boldsymbol u_2) = q - \lambda s^2 - \mu c^2$, to be non-negative, it is always possible to find a $q \in [\lambda s^2 + \mu c^2, \lambda_2)$ that will fulfil this requirement, provided $\boldsymbol{X} \neq \boldsymbol{X}_2$ with $\boldsymbol{u}_2 = (0, 1)^{\textup T} \perp \boldsymbol{u}_1$.
    The last term can be calculated as
    \begin{align*}
        \det(\boldsymbol{Y}) &= 
        \begin{vmatrix}
            1 - \lambda c^2 - \mu s^2 \hspace{0.5em}& cs(\lambda - \mu) \\
            cs(\lambda - \mu) \hspace{0.5em}& q - \lambda s^2 - \mu c^2
        \end{vmatrix} \\
        &= (1 - \lambda c^2 - \mu s^2)(q - \lambda s^2 - \mu c^2) - c^2 s^2 (\lambda - \mu)^2 
    \end{align*}
    \begin{align*}
        &= q(1 - \lambda c^2 - \mu s^2) - \lambda s^2 + c^2 s^2 (\lambda^2 + \mu^2) + (c^4 + s^4) \lambda \mu - \mu c^2  - c^2 s^2 (\lambda - \mu)^2 \\
        &= q(1 - \lambda c^2 - \mu s^2) - \lambda s^2 - \mu c^2 + (c^4 + 2 c^2 s^2 + s^4) \lambda \mu \\
        &= q(1 - \lambda c^2 - \mu s^2) - \lambda s^2 - \mu c^2 + (c^2 + s^2)^2 \lambda \mu \\
        &= q(1 - \lambda c^2 - \mu s^2) - \lambda s^2 - \mu c^2 + \lambda \mu
    \end{align*}
    which is non-negative if and only if
    \begin{align}
        q \geq \frac{\lambda s^2 + \mu c^2 - \lambda \mu}{1 - \lambda c^2 - \mu s^2} &= \lambda \frac{1 - \lambda c^2 - \mu s^2}{1 - \lambda c^2 - \mu s^2} + \frac{\lambda (s^2 - 1) + \mu c^2 - \lambda \mu (1 - s^2) + \lambda^2 c^2}{1 - \lambda c^2 - \mu s^2} \nonumber \\
        &= \lambda + \frac{c^2 (\mu - \lambda \mu + \lambda^2 - \lambda)}{1 - \lambda c^2 - \mu s^2} = \lambda + \frac{c^2 (1 - \lambda) (\mu - \lambda)}{1 - \lambda c^2 - \mu s^2} \nonumber \\
        &= \lambda - \frac{c^2 (1 - \lambda) (\lambda - \mu)}{1 - \lambda c^2 - \mu s^2}.
        \label{ineq_small_min_rhs}
    \end{align}
    Since $c^2 \geq 0$, $\mu \leq \lambda < 1$ and \eqref{ineq_pos_EW}, we see that the fraction on the right-hand side is non-negative. This means that the largest term that can occur on the right-hand side would be $\lambda$ and since $\lambda_2$ is our second largest eigenvalue and unique, $\lambda < \lambda_2$ except in the case where we consider the matrix $\boldsymbol{X}$ belonging to $\lambda_2$. However, in this case, because of uniqueness, $\lambda_2 > \mu_2$ and $c^2 \neq 0$, otherwise $\boldsymbol u_2$ would point in the same direction as $\boldsymbol v_1$, contradicting the condition $\boldsymbol{u}_2 \not\perp \boldsymbol{u}_1$. This means that the fraction on the right-hand side would be positive and thus the term on the right-hand side would be strictly less than $\lambda_2$.

    The case for $\boldsymbol{X}$ with $\lambda = \lambda_1$ provides the matrix
    \begin{align*}
        \boldsymbol{Y} = 
        \begin{pmatrix}
            0 \hspace{0.5em}& 0 \\
            0 \hspace{0.5em}& q - \mu_1
        \end{pmatrix}
    \end{align*}
    with $\mu_1 < \lambda_2$ because $\lambda_2$ is the second largest eigenvalue and unique. The relation $\boldsymbol{Y} \geq_\textup L \boldsymbol{0}$ holds if and only if $q \geq \mu_1$. Since $\mu_1 < \lambda_2$ and $q < \lambda_2$ (otherwise $\boldsymbol S' = \boldsymbol{S}$), we can find a $q$ with $\mu_1 \leq q < \lambda_2$ to fulfil the requirement.

    In summary, this means that in every case the right-hand side of \eqref{ineq_small_min_rhs} is strictly less than $\lambda_2$. Thus, if we compute for all $\boldsymbol X \in \mathcal X$ the bounds of $q$, which are all strictly less than $\lambda_2$, and then take the largest of them, we obtain the claim \eqref{smaller_LES} of the lemma.
    \hfill $\square$
\end{proof}

\begin{corollary}
    Under the conditions of Lemma 7, $\boldsymbol S$ cannot be the unique minimiser of a function $\varphi$ according to \eqref{phi}.
\end{corollary}
\begin{proof}
    Suppose $\boldsymbol S$ were a minimiser of a function $\varphi$ according to \eqref{phi}. Then $\varphi(\boldsymbol S) \geq \varphi(\boldsymbol S')$ because $\boldsymbol S \geq_\textup L \boldsymbol S'$ and the Loewner monotonicity of $\varphi$. Since $\boldsymbol S$ is minimal with respect to $\varphi$, only $\varphi(\boldsymbol S) = \varphi(\boldsymbol S')$ can be true. However, this is a contradiction to the uniqueness of the minimiser.
    \hfill $\square$
\end{proof}

This may seem like a big disadvantage at first. However, with the following slight modification of our set $\mathcal{U}(\mathcal{X})$ of possible upper bounds of $\mathcal{X}$, we can ensure that the LES actually acts as a minimiser there, as we will demonstrate in this section.

\begin{definition} 
    For a multi-set $\mathcal{X}$ of symmetric $2 \times 2$ matrices with nonnegative eigenvalues, let $\mathcal{X}^p := \{ \boldsymbol{X}^p : \boldsymbol{X} \in  \mathcal{X} \}$ be the element-wise application of the $p$-th power to the multi-set $\mathcal{X}$. We define the \textbf{$p$-power upper bound cone} as 
    \begin{align}
        \mathcal{U}_p(\mathcal{X}) := \big( \mathcal{U}(\mathcal{X}^p) \big)^\frac{1}{p} = \{ \boldsymbol{Y} \in \textup{Sym}(2) : \boldsymbol Y^p \in \mathcal{U}(\mathcal{X}^p) \}
        \label{p-power_upper_bound_cone}
    \end{align}
    and denote the intersection of all $p$-power upper bound cones as the \textbf{super-upper bound cone}
    \begin{align}
        \mathcal{U}_*(\mathcal{X}) := \bigcap_{p > 0} \mathcal{U}_p(\mathcal{X}).
        \label{super-upper_bound_cone}
    \end{align}
    For a multi-set $\mathcal{X}$ of symmetric $2\times2$ matrices for which $-c<0$ is the smallest among all eigenvalues of the matrices in $\mathcal{X}$, define
    \begin{align}
        \mathcal{U}_*(\mathcal{X}) := \,\mathcal{U}_*(\mathcal{X}+c) - c,
        \label{super-upper_bound_cone_(lifted)}
    \end{align}
    where 
    \begin{align*}
        \mathcal{X}+c := \{\boldsymbol{X}+c\boldsymbol{I} : \boldsymbol{X} \in\mathcal{X} \}
    \end{align*}
    and $\boldsymbol{I}$ is the identity matrix.
\end{definition}

\begin{remark}
    Equation $\eqref{super-upper_bound_cone_(lifted)}$ is particularly applicable to the set $\mathcal{U}(\mathcal{X})$ of upper bounds of a multi-set $\mathcal{X}$, as $\Sup_{\textup{LE}} (\mathcal{X}+c) - c\boldsymbol{I} = \Sup_{\textup{LE}}(\mathcal{X})$ and thus $\mathcal{U}(\mathcal{X}+c) - c = \mathcal{U}(\mathcal{X})$ always hold.
\end{remark}

To show that $\boldsymbol S$ is the unique minimiser of $\varphi$ in $\mathcal{U}_*(\mathcal{X})$, we first prove that it is contained in it.

\begin{lemma}
    Let $\mathcal{X}$ be a multi-set of symmetric $2 \times 2$ matrices, and let $\boldsymbol S$ be as in Lemma 7. Then one has $\boldsymbol S \in \mathcal{U}_p(\mathcal{X})$ for any $p > 0$.
\end{lemma}
\begin{proof}
    We use for $\boldsymbol S$ the same form as in Lemma 7 and assume again without loss of generality that $\lambda_1 = 1$, and its associated eigenvector is $\boldsymbol u_1 = (1,0)^\textup T$. This implies $\boldsymbol v_1 = (0,1)^\textup T$.
    Since we only consider real symmetric and thus diagonalisable matrices, we can simply express the power of such a matrix according to \cite{matrix_algebra} as follows:
    \begin{align*}
        \boldsymbol X^p = \lambda^p \boldsymbol u \boldsymbol u^\textup T + \mu^p \boldsymbol v \boldsymbol v^\textup T.
    \end{align*}
    The representation of $\mathcal{S}$ in terms of the $p$-th power will then take the following form:
    \begin{align*}
        \boldsymbol S^p = 
        \begin{pmatrix}
            1 \hspace{0.5em}& 0 \\
            0 \hspace{0.5em}& \lambda_2^p
        \end{pmatrix}.
    \end{align*}
    On the basis of $\mathcal{U}(\mathcal{X}+c) - c = \mathcal{U}(\mathcal{X})$, see Remark 3, we can assume without restriction that $\lambda_2 \geq 0$ applies.
    So we see that the only difference to our previous calculation of the LES is that all eigenvalues are now raised to the power of $p$. This means that our LES looks the same as before with the difference that we exponentiate its entries, i.e. the two largest permissible eigenvalues, by $p$ and thus obtain $\boldsymbol S^p$. Thus we have $\boldsymbol S^p \in \mathcal{U}(\mathcal{X}^p)$, from which the claim follows.
    \hfill $\square$
\end{proof} 

Next, we will define the as yet unspecified function $\varphi$ and show that it fulfils all the necessary conditions according to section 2.2. In particular, we will see that for this function $\boldsymbol S$ is the unique minimiser in $\mathcal{U}_*(\mathcal{X})$.

\begin{theorem}
    Let the conditions of Lemma 8 be satisfied and let the function $\boldsymbol \varphi: \mathcal{U}(\mathcal{X}) \rightarrow \mathbb R^2$ be explained as $\boldsymbol \varphi(\boldsymbol Y) = (\lambda,\mu)$, where $\lambda > \mu$ are the eigenvalues of $\boldsymbol Y \in \mathcal{U}(\mathcal{X})$. Further, let $\mathbb R^2$ be endowed with the lexicographic order 
    \begin{align}
        (a,b) \prec (a',b') :\Longleftrightarrow (a < a' \vee (a = a' \wedge b \leq b')), \quad a,b,a',b' \in \mathbb R.
        \label{lexicogrphic_order}
    \end{align}
    Then it follows that 
    \begin{itemize}
        \item[$(i)$] the function $\boldsymbol \varphi$ is Loewner-monotone, \vspace{0.2cm}
        \item[$(ii)$] the function $\boldsymbol \varphi$ is convex on $\mathcal{U}(\mathcal{X})$, and \vspace{0.2cm}
        \item[$(iii)$] $\boldsymbol{S}$ is the unique minimiser of $\boldsymbol \varphi$ in $\mathcal{U}_*(\mathcal{X})$.
    \end{itemize}
\end{theorem}
\begin{proof}
    $(i)$ Let $\boldsymbol A, \boldsymbol B \in \mathcal{U}(\mathcal{X})$ 
    with $\boldsymbol B \leq_\textup L \boldsymbol A$. The inequality implies that for any vector $\boldsymbol{w} \in \mathbb R^2$ the Rayleigh products must satisfy $R_{\boldsymbol{B}}(\boldsymbol{w}) \leq R_{\boldsymbol{A}}(\boldsymbol{w})$.

    By setting $\boldsymbol w = \boldsymbol u_{\boldsymbol B}$ to the unit eigenvector corresponding to the largest eigenvalue $\lambda_{\boldsymbol B}$ of $\boldsymbol B$, the Rayleigh product of $\boldsymbol B$ will equal the largest eigenvalue according to $R_{\boldsymbol{B}}(\boldsymbol u_{\boldsymbol{B}}) = \lambda_{\boldsymbol{B}}$. The Rayleigh product of $\boldsymbol A$ with $\boldsymbol u_{\boldsymbol B}$ will take a value $\alpha \in [\mu_{\boldsymbol A}, \lambda_{\boldsymbol{A}}]$ according to the min-max theorem (cf. \cite{min-max-theorem}, Theorem 4.2.2), where $\mu_{\boldsymbol{A}}$ is the smaller and $\lambda_{\boldsymbol{A}}$ the larger eigenvalue of $\boldsymbol{A}$. Upon inserting this into the inequality for the Rayleigh products, we obtain 
    \begin{align*}
        \lambda_{\boldsymbol{B}} = R_{\boldsymbol{B}}(\boldsymbol{u}_{\boldsymbol{B}}) \leq R_{\boldsymbol{A}}(\boldsymbol{u}_{\boldsymbol{B}}) = \alpha \leq \lambda_{\boldsymbol{A}},
    \end{align*}
    from which it follows that the largest eigenvalue of $\boldsymbol A$ is greater or equal to the largest eigenvalue of $\boldsymbol{B}$. If $\lambda_{\boldsymbol{A}} > \lambda_{\boldsymbol{B}}$, $\boldsymbol \varphi(\boldsymbol{B}) \preccurlyeq \boldsymbol \varphi(\boldsymbol{A})$ is already established.

    If this is not the case, we need to consider the minor eigenvalues of $\boldsymbol{A}$ and $\boldsymbol{B}$. By choosing $\boldsymbol{w} = \boldsymbol v_{\boldsymbol{A}}$ as the unit eigenvector corresponding to the smallest eigenvalue $\mu_{\boldsymbol{A}}$ of $\boldsymbol{A}$, the Rayleigh product of $\boldsymbol A$ will equal the smallest eigenvalue according to $R_{\boldsymbol{A}}(\boldsymbol v_{\boldsymbol{A}}) = \mu_{\boldsymbol{A}}$. If we also consider the Rayleigh product with $\boldsymbol{B}$, we again obtain with the min-max theorem a value $\beta \in [\mu_{\boldsymbol{B}}, \lambda_{\boldsymbol{B}}] $, where $\mu_{\boldsymbol{B}}$ is the smallest eigenvalue of $\boldsymbol{B}$, for this and by inserting it into the inequality mentioned at the beginning, we obtain
    \begin{align*}
        \mu_{\boldsymbol{B}} \leq \beta = R_{\boldsymbol{B}}(\boldsymbol v_{\boldsymbol{A}}) \leq R_{\boldsymbol{A}}(\boldsymbol v_{\boldsymbol{A}}) = \mu_{\boldsymbol{A}}.
    \end{align*}
    Consequently, we can guarantee that the smallest eigenvalue of $\boldsymbol{B}$ is less than or equal to the smallest eigenvalue of $\boldsymbol{A}$ and thus $\boldsymbol \varphi(\boldsymbol B) \preccurlyeq \boldsymbol \varphi(\boldsymbol A) $ holds. \vspace{0.2cm}

    \noindent $(ii)$ We want to prove here that $\boldsymbol \varphi$ is convex on $\mathcal{U}(\mathcal{X})$ according to Definition 3. For this we consider the two vector spaces $\textup{Sym}(2)$ and $\mathbb R^2$ and the convex set $\mathcal{U}(\mathcal{X}) \subset \textup{Sym}(2)$. We have already seen the convexity of the last set in section 2.2. Further, we need an order cone on the $\mathbb R^2$ equipped with the lexicographic order, which we denote simply as 
    \begin{align*}
        K := \{ \boldsymbol{x} \in \mathbb R^2 : \boldsymbol{x} \succcurlyeq \boldsymbol{0} \}.
    \end{align*}
    That is, the only thing we still need to check is the validity of \eqref{convexity}. For this we consider for all $\boldsymbol{Y}_1, \boldsymbol{Y}_2 \in \mathcal{U}(\mathcal{X})$ and $\alpha \in [0,1]$:
    \begin{alignat}{3}
        &\alpha \boldsymbol \varphi(\boldsymbol{Y}_1) + (1-\alpha) \boldsymbol \varphi(\boldsymbol{Y}_2) - \boldsymbol \varphi(\alpha \boldsymbol{Y}_1 + (1-\alpha) \boldsymbol{Y}_2) \in K \nonumber \\
        &\Longleftrightarrow \boldsymbol{0} \preccurlyeq \alpha \boldsymbol \varphi(\boldsymbol{Y}_1) + (1-\alpha) \boldsymbol \varphi(\boldsymbol{Y}_2) - \boldsymbol \varphi(\alpha \boldsymbol{Y}_1 + (1-\alpha) \boldsymbol{Y}_2) \nonumber \\
        & \Longleftrightarrow  \alpha \boldsymbol \varphi(\boldsymbol{Y}_1) + (1-\alpha) \boldsymbol \varphi(\boldsymbol{Y}_2) \succcurlyeq \boldsymbol \varphi(\alpha \boldsymbol{Y}_1 + (1-\alpha) \boldsymbol{Y}_2) =: \varphi(\boldsymbol{Z}).
        \label{lexi_ineq}
    \end{alignat}
    
    Let us assume without loss of generality that $\boldsymbol{Z}$ is a diagonal matrix. Further, let $\lambda_i \geq \mu_i$ be the eigenvalues of $\boldsymbol{Y}_i$ with the corresponding eigenvectors $(c_i, s_i)^\textup T, (-s_i, c_i)^\textup T$ for $i = 1,2$ according to \eqref{spectral_decomposition}. Then, we have
    \begin{align*}
        &\boldsymbol{Z} = 
        \begin{pmatrix}
            \lambda_* \hspace{0.5em}& 0 \\
            0 \hspace{0.5em}& \mu_*
        \end{pmatrix}
        \quad \text{with} \quad \lambda_* := \alpha (\lambda_1 c_1^2 + \mu_1 s_1^2) + (1-\alpha) (\lambda_2 c_2^2 + \mu_2 s_2^2) \\
        &\text{and} \quad \mu_* := \alpha (\lambda_1 s_1^2 + \mu_1 c_1^2) + (1-\alpha) (\lambda_2 s_2^2 + \mu_2 c_2^2).
    \end{align*}
    Note that the two eigenvalues $\lambda_*, \mu_*$ do not necessarily have to be ordered in this way. If we substitute this into our inequality \eqref{lexi_ineq}, we obtain
    \begin{align}
        \alpha 
        \begin{pmatrix}
            \lambda_1 \\
            \mu_1
        \end{pmatrix}
         + (1-\alpha) 
         \begin{pmatrix}
             \lambda_2 \\
             \mu_2
         \end{pmatrix}
         \succcurlyeq 
         \begin{pmatrix}
             \lambda_* \\
             \mu_*
         \end{pmatrix}.
         \label{lexi_ineq_2}
    \end{align}
    
    To prove this inequality, we first show that both eigenvalues $\lambda_*, \mu_*$ are less than or equal to $\alpha \lambda_1 + (1-\alpha) \lambda_2$:
    \begin{align*}
        &\lambda_* \leq \alpha (\lambda_1 c_1^2 + \lambda_1 s_1^2) + (1-\alpha) (\lambda_2 c_2^2 + \lambda_2 s_2^2) = \alpha \lambda_1 + (1-\alpha) \lambda_2, \\
        &\mu_* \leq \alpha (\lambda_1 s_1^2 + \lambda_1 c_1^2) + (1-\alpha) (\lambda_2 s_2^2 + \lambda_2 c_2^2) = \alpha \lambda_1 + (1-\alpha) \lambda_2.
    \end{align*}
    If strict inequality applies to both inequalities, we would have proven the statement. Otherwise, one of the two diagonal entries is equal to $\alpha \lambda_1 + (1-\alpha) \lambda_2$, without loss of generality, let this be $\lambda_*$. Then it follows
    \begin{align*}
        \mu_* &= \trace(\boldsymbol{Z}) - \lambda_* = \trace(\boldsymbol{Z}) - \alpha \lambda_1 - (1-\alpha) \lambda_2 \\
        &= \alpha \trace(\boldsymbol{Y}_1) + (1-\alpha) \trace(\boldsymbol{Y}_2) - \alpha \lambda_1 - (1-\alpha) \lambda_2 \\
        &= \alpha (\trace(\boldsymbol{Y}_1) - \lambda_1) + (1-\alpha) (\trace(\boldsymbol{Y}_2) - \lambda_2) \\
        &= \alpha (\lambda_1 + \mu_1 - \lambda_1) + (1-\alpha) (\lambda_2 + \mu_2 - \lambda_2) = \alpha \mu_1 + (1-\alpha) \mu_2
    \end{align*}
    whereby we have exploited in the last line that the trace of a matrix can be understood as the sum of its eigenvalues. This means that in this case the smaller eigenvalue of $\boldsymbol{Z}$ is just the convex combination of the smaller eigenvalues of $\boldsymbol{Y}_1$ and $\boldsymbol{Y}_2$ or, in other words, satisfies the second line of inequality \eqref{lexi_ineq_2} with equality. \vspace{0.2cm}

    \noindent $(iii)$ For this we use the argumentation from the proof of Lemma 7 for $(\boldsymbol S')^p$ and $\boldsymbol X^p$ instead of $\boldsymbol S'$ and $\boldsymbol X$ to obtain a necessary condition for $(\boldsymbol S')^p \geq_\textup L \boldsymbol X^p$, using the representation from the proof of Lemma 8 for $\boldsymbol X^p$. Since in both cases the exponentiation by $p$ does not change the eigenvectors but only the eigenvalues, we can take our conclusion \eqref{ineq_small_min_rhs} from the proof of Lemma 7 by exponentiating the eigenvalues all by $p$:
    \begin{align}
        q^p \geq \lambda^p - \frac{c^2 (1 - \lambda^p) (\lambda^p - \mu^p)}{1 - \lambda^p c^2 - \mu^p s^2}.
        \label{ineq_small_min_rhs_2}
    \end{align}
    As in the proof of Lemma 7, we have again that the fraction on the right-hand side of \eqref{ineq_small_min_rhs_2} is non-negative for all $\mu \leq \lambda < 1$ and $p>0$. 
    
    In the proof of Lemma 7 we had shown that the right-hand side of the inequality \eqref{ineq_small_min_rhs} analogous to \eqref{ineq_small_min_rhs_2} was always smaller than $\lambda_2$. Thus we could guarantee that a $\varepsilon > 0$ exists, so that a similar matrix to the LES $\boldsymbol{S}$ exists whose second eigenvalue according to $\lambda_2 - \varepsilon < \lambda_2$ is strictly smaller than that of the LES $\boldsymbol{S}$, which showed that the LES was not the unique minimiser. So all we have to prove here is that no such $\varepsilon$ exists for given $\lambda, \mu, c, s$ and $p \rightarrow \infty$.

    Let us assume that $\varepsilon>0$ exists so that \eqref{ineq_small_min_rhs_2} is satisfied for $q=\lambda-\varepsilon$ and arbitrarily large $p>0$. In particular, we can also assume without restriction of generality that our eigenvalues can be mapped bijectively to the interval $[0,1]$ or more precisely to 
    $\left[\frac{\frac{1}{\sqrt{2}} - \lambda_1^\text{old}}{\sqrt{2}}, 1 \right]$, 
    where $\lambda_1^\text{old}$ is the original largest eigenvalue of $\mathcal{X}$, so that $\lambda_1 = 1$ and $\mu_1,\lambda_i,\mu_i \in [0,1)$ for $i = 2,\dots,n$, see the following remark.
    This would mean that 
    \begin{align}
        &0 \leq (\lambda - \varepsilon)^p - \lambda^p + \frac{c^2 (1 - \lambda^p) (\lambda^p - \mu^p)}{1 - \lambda^p c^2 - \mu^p s^2} \nonumber \\
        &\Longleftrightarrow \lambda^p - (\lambda - \varepsilon)^p \leq \frac{c^2 (1 - \lambda^p) (\lambda^p - \mu^p)}{1 - \lambda^p c^2 - \mu^p s^2} \leq \frac{(1 - \lambda^p) (\lambda^p - \mu^p)}{1 - \lambda^p c^2 - \mu^p s^2} \nonumber \\
        &\Longleftrightarrow (\lambda - \varepsilon)^p \geq \lambda^p - \frac{(1 - \lambda^p) (\lambda^p - \mu^p)}{1 - \lambda^p c^2 - \mu^p s^2} \nonumber \\
        &\Longleftrightarrow \varepsilon \leq \lambda - \left[ \lambda^p - \frac{(1 - \lambda^p) (\lambda^p - \mu^p)}{1 - \lambda^p c^2 - \mu^p s^2} \right]^\frac{1}{p}, \quad 0 \leq \mu \leq \lambda < 1, \quad p>0,
        \label{ineq_small_min_rhs_3}
    \end{align}
    holds. 
    We rewrite the right-hand side of the inequality \eqref{ineq_small_min_rhs_3} into
    \begin{align}
    \begin{split}
        &\lambda - \left[ \lambda^p - \frac{(1 - \lambda^p) (\lambda^p - \mu^p)}{1 - \lambda^p c^2 - \mu^p s^2} \right]^\frac{1}{p} = \lambda - \left[ \lambda^p - \lambda^p  \frac{ \left( 1- \frac{\mu^p}{\lambda^p} \right) (1 - \lambda^p)}{1 - \lambda^p c^2 - \mu^p s^2} \right]^\frac{1}{p} \\
        &\quad = \lambda - \left[ \lambda^p \left( 1 - \left( 1- \frac{\mu^p}{\lambda^p} \right) \frac{(1 - \lambda^p)}{1 - \lambda^p c^2 - \mu^p s^2} \right) \right]^\frac{1}{p} \\
        &\quad = \lambda \left\{ 1 - \left[ 1 - \left( 1- \frac{\mu^p}{\lambda^p} \right) \frac{(1 - \lambda^p)}{1 - \lambda^p c^2 - \mu^p s^2} \right]^\frac{1}{p} \right\}.
    \end{split}
    \label{transform_of_rhs_of_ineq_small_min_rhs_3}
    \end{align}
    If we consider the last $[\,\cdot\,]^\frac{1}{p}$ term in the equation \eqref{transform_of_rhs_of_ineq_small_min_rhs_3}, we can estimate as follows
    \begin{align}
        \left[ 1 - \frac{1}{1-\lambda^p c^2 - \mu^p s^2} \right]^\frac{1}{p} \leq \left[ 1 - \left( 1- \frac{\mu^p}{\lambda^p} \right) \frac{(1 - \lambda^p)}{1 - \lambda^p c^2 - \mu^p s^2} \right]^\frac{1}{p} \leq 1^\frac{1}{p}
        \label{up_and_low_est_1}
    \end{align}
    and by using the limes for $p \rightarrow \infty$ we obtain for the upper and lower limit
    \begin{align}
        &\lim_{p \rightarrow \infty} 1^\frac{1}{p} = 1,
        \label{limit_of_up_est_1}
        \\
    \begin{split}
        &\lim_{p \rightarrow \infty} \left[ 1 - \frac{1}{1-\lambda^p c^2 - \mu^p s^2} \right]^\frac{1}{p} = \lim_{p \rightarrow \infty} \exp \left( \frac{1}{p} \log \left( 1 - \frac{1}{1-\lambda^p c^2 - \mu^p s^2} \right) \right) \\
        & \quad = \exp \left( \lim_{p \rightarrow \infty} \frac{1}{p} \log \left( 1 - \frac{1}{1-\lambda^p c^2 - \mu^p s^2} \right) \right) = \mathrm{e}^0 = 1.
    \end{split}
    \label{limit_of_low_est_1}
    \end{align}
    By combining \eqref{up_and_low_est_1} with \eqref{limit_of_up_est_1} and \eqref{limit_of_low_est_1}, we achieve 
    \begin{align*}
        1  \leq \lim_{p \rightarrow \infty} \left[ 1 - \left( 1- \frac{\mu^p}{\lambda^p} \right) \frac{(1 - \lambda^p)}{1 - \lambda^p c^2 - \mu^p s^2} \right]^\frac{1}{p} \leq   1
    \end{align*}
    or 
    \begin{align}
       \lim_{p \rightarrow \infty} \left[ 1 - \left( 1- \frac{\mu^p}{\lambda^p} \right) \frac{(1 - \lambda^p)}{1 - \lambda^p c^2 - \mu^p s^2} \right]^\frac{1}{p} = 1.
       \label{limit_of_ineq_small_min_rhs_3}
    \end{align}
    For the last step, we substitute \eqref{transform_of_rhs_of_ineq_small_min_rhs_3} in \eqref{ineq_small_min_rhs_3}, let $p$ approach infinity and use the equation \eqref{limit_of_ineq_small_min_rhs_3}:
    \begin{align*}
        \varepsilon &\leq \lim_{p \rightarrow \infty} \lambda \left\{ 1 - \left[ 1 - \left( 1- \frac{\mu^p}{\lambda^p} \right) \frac{(1 - \lambda^p)}{1 - \lambda^p c^2 - \mu^p s^2} \right]^\frac{1}{p} \right\} \\
        &\quad = \lambda \left\{ 1 - \lim_{p \rightarrow \infty} \left[ 1 - \left( 1- \frac{\mu^p}{\lambda^p} \right) \frac{(1 - \lambda^p)}{1 - \lambda^p c^2 - \mu^p s^2} \right]^\frac{1}{p} \right\} = 0,
    \end{align*}
    which is a contradiction to $\varepsilon > 0$. This means that $\boldsymbol{S}' \notin \mathcal{U}_*(\mathcal{X})$. \hfill $\square$
\end{proof}

\begin{remark}
    Let $\boldsymbol X \in \mathcal{X}$ be arbitrary and $\boldsymbol{X} = \lambda \boldsymbol{u} \boldsymbol{u}^\textup T + \mu \boldsymbol{v} \boldsymbol{v}^\textup T$ as in \eqref{spectral_decomposition}, where $\lambda,\mu$ are the eigenvalues of $\boldsymbol{X}$ and $\boldsymbol{u},\boldsymbol{v}$ the associated normalised eigenvectors. Since the greatest element in the HCL bi-cone is the colour white with the eigenvalues $\frac{1}{\sqrt{2}}$ and the smallest element is the colour black with the eigenvalues $- \frac{1}{\sqrt{2}}$, see \cite{WelkQuantile}, we have $\lambda,\mu \in \left[- \frac{1}{\sqrt{2}}, \frac{1}{\sqrt{2}}\right]$. This means that the transformation
    \begin{align*}
        \bar{\boldsymbol{X}} &:= \frac{1}{\sqrt{2}} \boldsymbol{X} + \left( 1 - \frac{\lambda_1}{\sqrt{2}} \right) \boldsymbol{I} \\
        &= \frac{1}{\sqrt{2}} \lambda \boldsymbol{u} \boldsymbol{u}^\textup T + \frac{1}{\sqrt{2}} \mu \boldsymbol{v} \boldsymbol{v}^\textup T + \left( 1 - \frac{\lambda_1}{\sqrt{2}} \right) \boldsymbol{u} \boldsymbol{u}^\textup T + \left( 1 - \frac{\lambda_1}{\sqrt{2}} \right) \boldsymbol{v} \boldsymbol{v}^\textup T \\
        &= \left( \frac{\lambda}{\sqrt{2}} + 1 - \frac{\lambda_1}{\sqrt{2}} \right) \boldsymbol{u} \boldsymbol{u}^\textup T + \left( \frac{\mu}{\sqrt{2}} + 1 - \frac{\lambda_1}{\sqrt{2}} \right) \boldsymbol{v} \boldsymbol{v}^\textup T =: \Bar{\lambda} \boldsymbol{u} \boldsymbol{u}^\textup T + \Bar{\mu} \boldsymbol{v} \boldsymbol{v}^\textup T 
    \end{align*}
    has the eigenvalues $\Bar{\lambda},\Bar{\mu} \in [0,1]$. 
    It follows from the fact that 
    \begin{align*}
        1 + \frac{\lambda - \lambda_1}{\sqrt{2}} \in \left[\frac{\frac{1}{\sqrt{2}} - \lambda_1}{\sqrt{2}}, 1 \right] \subseteq [0,1],
    \end{align*}
    where the lower and upper bound are derived from $\lambda \in \left[-\frac{1}{\sqrt{2}}, \lambda_1 \right] $.
    Consequently, the addition of the scaled unit matrix affects only the eigenvalues but not the eigenvectors. 
    As a result, this can also be applied directly to the LES:
    \begin{align*}
        \bar{\boldsymbol{S}} := \Sup_{\textup{LE}}\left(\bar{\boldsymbol{X}}_1, \dots, \bar{\boldsymbol{X}}_n\right) = \frac{1}{\sqrt{2}} \boldsymbol{S} + \left( 1 - \frac{\lambda_1}{\sqrt{2}} \right) \boldsymbol{I}.
    \end{align*}
\end{remark}

\section{Relaxation of the LES}

Now that we have seen several favourable properties of the LES, we want to address a disadvantage of our approach that we have not considered further so far. It is one of the reasons why this approach was not pursued further in favour of the trace-supremum $\Sup_{\tr}$ in the work \cite{morph_op_mat_im}. This is due to the fact that the LES does not depend continuously on the input data. In this section we will first look at the cases where this problem occurs and then see that it can be solved using a straightforward relaxation.

To visualise this problem, let us consider the following scenario: We have a multi-set $\mathcal{X} = \{\boldsymbol{X}_1, \dots, \boldsymbol{X}_n\}$, $n \geq 3$, of symmetric $2 \times 2$ matrices in which the three largest eigenvalues $\lambda_1 > \lambda_2 > \lambda_3$ with the associated eigenvectors $\boldsymbol{u}_1, \boldsymbol{u}_2, \boldsymbol{u}_3$ belong to three different matrices $\boldsymbol{X}_1, \boldsymbol{X}_2, \boldsymbol{X}_3 \in \mathcal{X}$. In addition, the eigenvectors have the relationship $\boldsymbol{u}_1 = \boldsymbol{u}_2$ and $\boldsymbol{u}_3 \neq \pm \boldsymbol{u}_1$ to each other. According to Theorem 1, we obtain for the LES
\begin{align*}
    \boldsymbol{S} = \lambda_1 \boldsymbol{u}_1 \boldsymbol{u}_1^\textup T + \lambda_3 \boldsymbol{v}_1 \boldsymbol{v}_1^\textup T, \quad \boldsymbol{v}_1 \perp \boldsymbol{u}_1.
\end{align*}
We will now consider a rotated multi-set $\mathcal{X}^{(\delta)} = \{\boldsymbol{X}_1^{(\delta)}, \dots, \boldsymbol{X}_n^{(\delta)}\}$ by rotating the matrices $\boldsymbol{X}_i$ around the angles $\delta \alpha_i$ for $i \geq 2$ with the factor $\delta > 0$. The angles $\alpha_i$ can be different but $\sin(\delta \alpha_2) \neq 0$, such that the new eigenvector 
\begin{align*}
    \boldsymbol{u}_2^{(\delta)} = \boldsymbol{R}_{\delta \alpha_2} \boldsymbol{u}_2 = 
    \begin{pmatrix}
        \cos(\delta \alpha_2 ) \hspace{0.5em}& -\sin(\delta \alpha_2) \\
        \sin(\delta \alpha_2) \hspace{0.5em}& \cos(\delta \alpha_2 )
    \end{pmatrix}
    \begin{pmatrix}
        1 \\
        0
    \end{pmatrix}
    = 
    \begin{pmatrix}
        \cos(\delta \alpha_2) \\
        \sin(\delta \alpha_2 )
    \end{pmatrix}
    \neq \pm \boldsymbol{u}_1.
\end{align*}
As these are only rotations, the eigenvalues $\lambda_i^{(\delta)}, \mu_i^{(\delta)}$ of $\boldsymbol{X}_i^{(\delta)}$ according $\lambda_i^{(\delta)} = \lambda_i$ and $\mu_i^{(\delta)} = \mu_i$ are retained for all $i \in \{ 1, \dots, n \}$ and $\delta > 0$. By also applying Theorem 1 to this, we obtain
\begin{align*}
    \boldsymbol{S}^{(\delta)} := \Sup_{\textup{LE}}(\mathcal{X}^{(\delta)}) = \lambda_1 \boldsymbol{u}_1 \boldsymbol{u}_1^\textup T + \lambda_2 \boldsymbol{v}_1 \boldsymbol{v}_1^\textup T =: \boldsymbol{S}' \neq \boldsymbol{S}.
\end{align*}
If we now want to determine the limit value of this for $\delta \rightarrow 0$, we go to step 2 of the proof of Lemma 4 and obtain $\overline{R}(\varepsilon, \delta)$ according to Equation \eqref{leading terms} with $(s_2, c_2)^\textup T = \boldsymbol{u}_2^{(\delta)}$. Then, we calculate
\begin{align*}
    &\lim_{\delta \rightarrow 0} \lim_{m \rightarrow \infty} \left( \frac{1}{m} \log(\overline{R}(\varepsilon, \delta)) \right) \\
    & \quad = \lim_{\delta \rightarrow 0} \lim_{m \rightarrow \infty} \Bigg( \lambda_2 - \frac{\log (1 + \varepsilon^2)}{m} + \frac{\log \big( (-\sin(\delta \alpha_2) - \varepsilon \cos(\delta \alpha_2 ))^2\big)}{m} \\
    & \qquad + \frac{\log \left( 1 + \mathcal{O}\left( \mathrm{e}^{m(\lambda_2 - \lambda_1)} + \mathrm{e}^{m(\mu - \lambda_2)} \right) \right)}{m} \Bigg) \\
    & \quad = \lim_{\delta \rightarrow 0} \lambda_2 = \lambda_2,
\end{align*}
because $\varepsilon \rightarrow 0$ for $m \rightarrow \infty$, see proof of Lemma 4, and $\sin(\delta \varphi) \neq 0$. This leads to
\begin{align*}
    \lim_{\delta \rightarrow 0} \boldsymbol{S}^{(\delta)} = \boldsymbol{S}' \neq \boldsymbol{S}
\end{align*}
although $\mathcal{X}^{(\delta)} \rightarrow \mathcal{X}$ for $\delta \rightarrow 0$.

These types of discontinuity are removable,  since they only occur in non-generic configurations of $\mathcal{X}$. To address these kind of issues, we introduce the following definition.

\begin{definition}
    Let $\mathcal{X} = \{\boldsymbol{X}_1, \dots, \boldsymbol{X}_n\}$, $n \in \mathbb N$, be a multi-set of symmetric real $2 \times 2$ matrices with the spectral decomposition \eqref{spectral_decomposition}. 
    We call the multi-set 
    \begin{align*}
        \mathcal{B}(\mathcal{X}) := \{ \mathcal{B}(\boldsymbol{X}_1), \dots, \mathcal{B}(\boldsymbol{X}_n) \} := \big\{ \{ \boldsymbol{u}_1, \boldsymbol{v}_1 \}, \dots, \{ \boldsymbol{u}_n, \boldsymbol{v}_n \} \big\}
    \end{align*}
    of orthonormal bases of $\mathcal{X}$ \textbf{generic} if and only if no two orthonormal bases $\mathcal{B}(\boldsymbol{X}_i)$, $\mathcal{B}(\boldsymbol{X}_j)$, $i \neq j$, share the same orientation, i.e. $\boldsymbol{u}_i$ is not aligned with $\boldsymbol{u}_j$. 
    Furthermore, we define that \textbf{the multi-sets $\mathcal{X}^{(\delta)} = \big\{\boldsymbol{X}_1^{(\delta)}, \dots, \boldsymbol{X}_n^{(\delta)} \big\}$ converge planar towards the multi-set $\mathcal{X}$}, symbolically $\mathcal{X}^{(\delta)} \xlongrightarrow{\textup{pl}}	\mathcal{X}$, if $\mathcal{X}^{(\delta)} \rightarrow \mathcal{X}$ and the eigenvalues of $\boldsymbol{X}_i^{(\delta)}$ and $\boldsymbol{X}_i$ remain the same for all $i \in \{ 1, \dots, n \}$ and $\delta$.
    Then, we declare the \textbf{relaxed log-exp-supremum (RLES)} as
    \begin{align}
    \begin{split}
        \Sup_{\textup{RLE}}(\mathcal{X}) := 
        \begin{cases}
            \Sup_{\textup{LE}}(\mathcal{X}), & \textup{if } \mathcal{B}(\mathcal{X}) \textup{ is generic}, \\
            \lim\limits_{\substack{\mathcal{X}^{(\delta)} \xlongrightarrow{\textup{pl}} \mathcal{X} \\ \mathcal{B}(\mathcal{X}^{(\delta)}) \textup{ generic}}} \Sup_{\textup{LE}}(\mathcal{X}^{(\delta)}), & \textup{otherwise}.
        \end{cases}
    \end{split}
    \label{RLES}
    \end{align}
\end{definition}

\begin{remark}
    The idea of planar convergence stems from the fact that for 
    \begin{align*}
        &\boldsymbol{X} = \lambda \boldsymbol{u} \boldsymbol{u}^{\textup T} + \mu \boldsymbol{v} \boldsymbol{v}^{\textup T} = 
        \begin{pmatrix}
            \lambda u_1^2 + \mu v_1^2 \hspace{0.5em}& \lambda u_1 u_2 + \mu v_1 v_2 \\
            \lambda u_1 u_2 + \mu v_1 v_2 \hspace{0.5em}& \lambda u_2^2 + \mu v_2^2
        \end{pmatrix}
        , \\ 
        &\abs{\boldsymbol{u}} = 1 = \abs{\boldsymbol{v}}, \quad \boldsymbol{u} =
        \begin{pmatrix}
            u_1 \\
            u_2
        \end{pmatrix}
        , \boldsymbol{v} = 
        \begin{pmatrix}
            v_1 \\
            v_2
        \end{pmatrix}
        \in \mathbb{R}^2, \quad \lambda, \mu \in \mathbb R,
    \end{align*}
    and
    \begin{align*}
        &\boldsymbol{Y} = \lambda \boldsymbol{s} \boldsymbol{s}^{\textup T} + \mu \boldsymbol{t} \boldsymbol{t}^{\textup T} = 
        \begin{pmatrix}
            \lambda s_1^2 + \mu t_1^2 \hspace{0.5em}& \lambda s_1 s_2 + \mu t_1 t_2 \\
            \lambda s_1 s_2 + \mu t_1 t_2 \hspace{0.5em}& \lambda s_2^2 + \mu t_2^2
        \end{pmatrix}
        , \\ 
        &\abs{\boldsymbol{s}} = 1 = \abs{\boldsymbol{t}}, \quad \boldsymbol{s} =
        \begin{pmatrix}
            s_1 \\
            s_2
        \end{pmatrix}
        , \boldsymbol{t} = 
        \begin{pmatrix}
            t_1 \\
            t_2
        \end{pmatrix}
        \in \mathbb{R}^2, 
    \end{align*}
    their $z$-component in the bi-cone coordinates can be calculated as follows according to Burgeth and Kleefeld, see \cite{BurgethKleefeld}:
    \begin{align*}
        &z(\boldsymbol{X}) = \frac{\trace(\boldsymbol{X})}{\sqrt{2}}  = \frac{\lambda u_1^2 + \mu v_1^2 + \lambda u_2^2 + \mu v_2^2 }{\sqrt{2}}  = \frac{\lambda \abs{\boldsymbol{u}}^2 + \mu \abs{\boldsymbol{v}}^2}{\sqrt{2}} = \frac{\lambda + \mu}{\sqrt{2}}, \\
        &z(\boldsymbol{Y}) = \frac{\trace(\boldsymbol{Y})}{\sqrt{2}}  = \frac{\lambda s_1^2 + \mu t_1^2 + \lambda s_2^2 + \mu t_2^2 }{\sqrt{2}}  = \frac{\lambda \abs{\boldsymbol{s}}^2 + \mu \abs{\boldsymbol{t}}^2}{\sqrt{2}} = \frac{\lambda + \mu}{\sqrt{2}}.
    \end{align*}
    This means that the $z$-component remains the same for all colour matrices as long as they have the same eigenvalues, and that all the convergence happens on a particular $z$-plane.
\end{remark}

We can immediately deduce from this definition by means of the following corollary that the RLES does not leave the bi-cone.

\begin{corollary}
    Let $\mathcal{X} = (\boldsymbol{X}_1, \dots, \boldsymbol{X}_n)$, $n \in \mathbb N$, be a multi-set of symmetric real $2\times2$ matrices, which satisfy $- \frac{1}{\sqrt{2}} \boldsymbol{I} \leq_\textup L \boldsymbol{X}_i \leq_\textup L \frac{1}{\sqrt{2}} \boldsymbol{I}$ for all $i \in \{ 1, \dots, n \}$. Then we have 
    \begin{align*}
        - \frac{1}{\sqrt{2}} \boldsymbol{I} \leq_\textup L \Sup_{\textup{RLE}}(\mathcal{X}) \leq_\textup L \frac{1}{\sqrt{2}} \boldsymbol{I}.
    \end{align*}
\end{corollary}
\begin{proof}
    The assertion follows from the fact that the LES lies in the bi-cone and the RLES has the same eigenvalues for selection as the LES due to Definition 7. As a result, the largest eigenvalue cannot be greater than $\frac{1}{\sqrt{2}}$, which gives rise to the upper bound. The lower bound results from the transitivity of the Loewner order via $- \frac{1}{\sqrt{2}} \boldsymbol{I} \leq_\textup L \boldsymbol{X}_i \leq_\textup L \Sup_{\textup{RLE}}(\mathcal{X})$ for all $i \in \{ 1, \dots, n \}$.
    \hfill $\square$
\end{proof}

Following our first two previous theorems, we therefore also give a characterisation for the RLES in this form.

\begin{theorem}
    Let $\mathcal{X}$ be a multi-set of symmetric real $2 \times 2$ matrices. Further, let $\lambda_1$ be (one of) the largest eigenvalue(s) of $\mathcal{X}$, and $\boldsymbol{u}_1$ the eigenvector associated with the eigenvalue $\lambda_1$ and $\boldsymbol{v}_1 \perp \boldsymbol{u}_1$ the minor eigenvector of the same matrix from $\mathcal{X}$, and let $\lambda_2 \leq \lambda_1$ be the next largest eigenvalue of $\mathcal{X}$. Then the RLES of $\mathcal{X}$ is characterised by
    \begin{align}
    \begin{split}
        \Sup_{\textup{RLE}}(\mathcal{X}) = 
        \begin{cases}
            \lambda_1 \boldsymbol{u}_1 \boldsymbol{u}_1^\textup T + \lambda_2 \boldsymbol{v}_1 \boldsymbol{v}_1^\textup T, & \textup{if } \lambda_1 \textup{ is unique}, \\
            \lambda_1 \boldsymbol{I}, & \textup{otherwise,}
        \end{cases}
    \end{split}
    \label{RLES_characterisation}
    \end{align}
    where $\boldsymbol{I}$ is the identity matrix.
\end{theorem}
\begin{proof}
    The generic case of \eqref{RLES} is a straightforward consequence of Theorem 1 and Theorem 2. So we only need to consider the non-generic case. 
    
    Let $\lambda_1$ be unique. In this case we just need to follow the argumentation of the calculation that we used to introduce this section. By the rotation of eigenvectors that were shown there and the fact that the $\mathcal{X}^{(\delta)}$ of Definition 7 fulfils all necessary properties of the mentioned rotated multi-set, we can conclude the first case of Equation \eqref{RLES_characterisation}. 

    Now we want $\lambda_1$ to be not unique. This means that we have 
    \begin{align}
        \lambda_1 = \lambda_2 = \lambda_1^{(\delta)} = \lambda_2^{(\delta)}, 
        \label{planarity_of_eigenvalues}
    \end{align}
    where $\lambda_1^{(\delta)}, \lambda_2^{(\delta)}$ are the largest eigenvalues of $\mathcal{X}^{(\delta)}$ from Definition 7, since the $\mathcal{X}^{(\delta)}$ converge planar towards $\mathcal{X}$. Furthermore, let the corresponding eigenvectors of $\lambda_1, \lambda_2$ be given by $\boldsymbol{u}_1 = \boldsymbol{u}_2$. 
    We can find multi-sets $\mathcal{X}^{(\delta)}$ for $\delta > 0$ which have the same eigenvalues as $\mathcal{X}$ for all $\delta$, generic $\mathcal{B}\left( \mathcal{X}^{(\delta)} \right)$ and converge to $\mathcal{X}$ for $\delta \rightarrow 0$, e.g. by rotating the eigenvectors of $\mathcal{X}$ as at the beginning of this section.
    By Theorem 2 and Equation \eqref{planarity_of_eigenvalues} we get for the LES of $\mathcal{X}^{(\delta)}$:
    \begin{align*}
        \boldsymbol{S}^{(\delta)} := \Sup_{\textup{LE}}(\mathcal{X}^{(\delta)}) = \lambda_1^{(\delta)} \boldsymbol{I} = \lambda_1 \boldsymbol{I} \quad \forall \delta > 0.
    \end{align*}
    This means $\boldsymbol{S}^{(\delta_1)} = \boldsymbol{S}^{(\delta_2)}$ for all $\delta_1, \delta_2 > 0$. For $\delta \rightarrow 0$, it follows
    \begin{align*}
        \lim_{\delta \rightarrow 0} \boldsymbol{S}^{(\delta)} = \lim_{\delta \rightarrow 0} \lambda_1 \boldsymbol{I} = \lambda_1 \boldsymbol{I},
    \end{align*}
    which concludes the proof. 
    \hfill $\square$
\end{proof}

However, the question remains as to whether there is another type of problem that can cause the aforementioned discontinuity. We consider the following lemma to partially answer this question.

\begin{lemma}
    Let the conditions of Theorem 5 be fulfilled. Then, the RLES, given by Definition 7, depends continuously on the input data.
\end{lemma}
\begin{proof}
    Since the LES depends only on the spectral decomposition of the input data, the discontinuities can only occur through transformations of the input data, namely the eigenvalues and eigenvectors. Moreover, these transformations must remain in the bi-cone, so that only a scaling of the eigenvalues in the interval $\left[-\frac{1}{\sqrt{2}},\frac{1}{\sqrt{2}}\right]$ and a rotation of the eigenvectors are possible. We have already shown that the RLES depends continuously on the rotated input data. Therefore, it is sufficient to consider only the case of the scaled eigenvalues.

    Since we only want to investigate the case of scaled eigenvalues, we can assume that the eigenvectors will not change in the modified multi-set $\mathcal{Y}^{(\delta)} = \left\{\boldsymbol{Y}_1^{(\delta)}, \dots, \boldsymbol{Y}_n^{(\delta)} \right\} \rightarrow \mathcal{X}$ with $\boldsymbol{Y}_i^{(\delta)} = \lambda_i^{(\delta)} \boldsymbol{u}_i^{(\delta)} \left( \boldsymbol{u}_i^{(\delta)} \right)^{\textup T} + \mu_i^{(\delta)} \boldsymbol{v}_i^{(\delta)} \left( \boldsymbol{v}_i^{(\delta)} \right)^{\textup T}$ according to $\boldsymbol{u}_i^{(\delta)} = \boldsymbol{u}_i$ and $\boldsymbol{v}_i^{(\delta)} = \boldsymbol{v}_i$ for all $i \in \{ 1, \dots, n \}$. As a result, as long as the order of the eigenvalues sorted by size does not change, no discontinuities can occur with regard to the input data, since $\lambda_i^{(\delta)} \rightarrow \lambda_i$ and $\mu_i^{(\delta)} \rightarrow \mu_i$ for all $i \in \{ 1, \dots, n \}$. 
    
    For a multiplicative scaling in the form of $\lambda_i^{(\delta)} := \delta \lambda_i$ for $\delta \rightarrow 1$, there will not occur any discontinuity, since the order of eigenvalues would stay the same. If we assume an additive scaling in form of $\lambda_i^{(\delta)} := \lambda_i + \delta \alpha_i$, $\alpha_i \in \mathbb R$, for $\delta \rightarrow 0$, we could change the order of the eigenvalues, e.g. according to $\lambda_1 > \lambda_2 > \lambda_3$ but $\lambda_3^{(\delta)} > \lambda_1^{(\delta)} > \lambda_2^{(\delta)}$ for $\delta \in [a, b] \subset [\varepsilon_1, \infty)$ with a sufficiently small $\varepsilon_1 > 0$. Since the RLES operates only with generic bases, $\delta \rightarrow 0$ and $\lambda_1^{(\delta)} > \lambda_2^{(\delta)} > \lambda_3^{(\delta)}$ for $\delta \in (0,\varepsilon_2)$ with a sufficiently small $\varepsilon_2 < \varepsilon_1$, we would obtain as RLES
    \begin{align*}
        \lim_{\delta \rightarrow 0} \Sup_{\textup{RLE}}(\mathcal{Y}^{(\delta)}) &= \lim_{\substack{\delta \rightarrow 0 \\ \delta \in (0,\varepsilon_2)}} \lambda_1^{(\delta)} \boldsymbol{u}_1 \boldsymbol{u}_1^\textup T + \lambda_2^{(\delta)} \boldsymbol{v}_1 \boldsymbol{v}_1^\textup T = \lambda_1 \boldsymbol{u}_1 \boldsymbol{u}_1^\textup T + \lambda_2 \boldsymbol{v}_1 \boldsymbol{v}_1^\textup T \\
        &= \Sup_{\textup{RLE}}(\mathcal{X}).
    \end{align*}
    An analogous reasoning is applicable in the case of non-unique eigenvalues.
    \hfill $\square$
\end{proof}

We would now like to summarise the most important features of the RLES in the following corollary.

\begin{corollary}
    Let the conditions of Theorem 5 be fulfilled. Then, the RLES \eqref{RLES} can be characterised by \eqref{RLES_characterisation}, depends continuously on its input data and is transitive.
\end{corollary}
\begin{proof}
    The first two properties are covered by Theorem 5 and Lemma 9. 
    The eigenvectors in \eqref{Transitivity} can be rotated, such that the final input data is represented by generic orthonormal bases. Consequently, this constitutes a special case for input data with generic orthonormal bases in Proposition 1, which implies the transitivity.
    \hfill $\square$
\end{proof}

To conclude this section, we will consider a small example for the RLES with originally non-generic input data.

\begin{example}
    We calculate the RLES for the input data of Example 1, namely blue $\boldsymbol C_1 = (0,0,1)$, a medium-dark brown $\boldsymbol C_2 = \left( \frac{3}{5}, \frac{2}{5}, \frac{1}{5} \right)$ and a shade of blue-magenta $\boldsymbol C_3 = \left(\frac{1}{3}, \frac{1}{3}, \frac{5}{6}\right)$ with $\boldsymbol{X}_i = \lambda_i \boldsymbol{u}_i \boldsymbol{u}_i^\textup T + \mu_i \boldsymbol{v}_i \boldsymbol{v}_i^\textup T$, $i = 1, 2, 3$, and the corresponding spectral data
    \begin{align*}
        &\lambda_1 = \frac{1}{\sqrt{2}} \approx 0.7071,  &&\boldsymbol u_1 = \frac{1}{\sqrt{8 + 4 \sqrt{3}}} 
        \begin{pmatrix}
            -2-\sqrt{3} \\
            1
        \end{pmatrix}
        \approx
        \begin{pmatrix}
            -0.9659 \\
            \hspace{0.8em}0.2588
        \end{pmatrix}
        ,\\ 
        &\mu_1 = - \lambda_1,  &&\boldsymbol v_1 = \frac{1}{\sqrt{8 - 4 \sqrt{3}}}
        \begin{pmatrix}
            2-\sqrt{3} \\
            1
        \end{pmatrix}
        \approx
        \begin{pmatrix}
            -0.2588 \\
            -0.9659
        \end{pmatrix}
        , \\
        & \lambda_2 = \frac{1}{5 \sqrt{2}} \approx 0.1414,  &&\boldsymbol u_2 = \frac{1}{2} 
        \begin{pmatrix}
            1 \\
            \sqrt{3}
        \end{pmatrix}
        \approx
        \begin{pmatrix}
            0.5000 \\
            0.8660
        \end{pmatrix}
        ,\\
        &\mu_2 = - \frac{3}{5 \sqrt{2}} \approx -0.4243,  &&\boldsymbol v_2 = \frac{1}{2}
        \begin{pmatrix}
            -\sqrt{3} \\
            1
        \end{pmatrix}
        \approx
        \begin{pmatrix}
            -0.8660 \\
            \hspace{0.8em}0.5000
        \end{pmatrix}
        \quad \text{and} 
    \end{align*}
    \begin{align*}
        & \lambda_3 = \frac{2}{3 \sqrt{2}} \approx 0.4714,  &&\boldsymbol u_3 = \frac{1}{\sqrt{8 + 4 \sqrt{3}}}
        \begin{pmatrix}
            -2-\sqrt{3} \\
            1
        \end{pmatrix}
        \approx
        \begin{pmatrix}
            -0.9659 \\
            \hspace{0.8em}0.2588
        \end{pmatrix}
        , \\ 
        &\mu_3 = - \frac{1}{3 \sqrt{2}} \approx - 0.2357,  &&\boldsymbol v_3 = \frac{1}{\sqrt{8 - 4 \sqrt{3}}}
        \begin{pmatrix}
            2-\sqrt{3} \\
            1
        \end{pmatrix}
        \approx
        \begin{pmatrix}
            0.2588 \\
            0.9659
        \end{pmatrix}
        .
    \end{align*}
    Since $\mathcal{B}(\mathcal{X})$ is non-generic for $\mathcal{X} = \{ \boldsymbol{X}_1, \boldsymbol{X}_2, \boldsymbol{X}_3 \}$ because of $\boldsymbol{X}_1$ and $\boldsymbol{X}_3$, we consider a modified multi-set $\mathcal{X}^{(\delta)} = \{ \boldsymbol{X}_1^{(\delta)}, \boldsymbol{X}_2^{(\delta)}, \boldsymbol{X}_3^{(\delta)} \}$ according to Definition 7, which results by rotating the corresponding eigenvectors $\boldsymbol{u}_3, \boldsymbol{v}_3$ of the matrix $\boldsymbol{X}_3$ of $\boldsymbol{C}_3$ by $\delta \in (0, \pi)$:
    \begin{align*}
        &\boldsymbol{X}_1^{(\delta)} = \boldsymbol{X}_1, \quad \boldsymbol{X}_2^{(\delta)} = \boldsymbol{X}_2 \quad \text{and} \quad \boldsymbol{X}_3^{(\delta)} = \lambda_3^{(\delta)} \boldsymbol{u}_3^{(\delta)} \left( \boldsymbol{u}_3^{(\delta)} \right)^{\textup T} + \mu_3^{(\delta)} \boldsymbol{v}_3^{(\delta)} \left( \boldsymbol{v}_3^{(\delta)} \right)^{\textup T},\\
        & \lambda_3^{(\delta)} = \lambda_3, \quad \mu_3^{(\delta)} = \mu_3, \quad \boldsymbol u_3^{(\delta)} = \frac{1}{\sqrt{8 + 4 \sqrt{3}}}
        \begin{pmatrix}
            \left(-2-\sqrt{3}\right) \cos(\delta) - \sin(\delta) \\
            \left(-2-\sqrt{3}\right) \sin(\delta) + \cos(\delta)
        \end{pmatrix}
        , \quad \\
        & \boldsymbol v_3^{(\delta)} = \frac{1}{\sqrt{8 - 4 \sqrt{3}}}
        \begin{pmatrix}
            \left( 2-\sqrt{3} \right) \cos(\delta) - \sin(\delta) \\
            \left( 2-\sqrt{3} \right) \sin(\delta) + \cos(\delta)
        \end{pmatrix}
        , \quad \delta \in (0, \pi).
    \end{align*}
    Due to the generic nature of $\mathcal{B}\left(\mathcal{X}^{(\delta)}\right)$, we obtain for the RLES according to Theorem 5:
    \begin{align*}
        \boldsymbol{S}' = \lambda_1 \boldsymbol{u}_1 \boldsymbol{u}_1^\textup T + \lambda_3 \boldsymbol{v}_1 \boldsymbol{v}_1^\textup T = \frac{1}{12 \sqrt{2}}
        \begin{pmatrix}
            10+\sqrt{3} \hspace{0.5em}& -1 \\
            -1 \hspace{0.5em}& 10-\sqrt{3}
        \end{pmatrix}
        \approx 
        \begin{pmatrix}
            \hspace{0.3cm} 0.6913 \hspace{0.5em}& -0.0589 \\
            -0.0589 \hspace{0.5em}& \hspace{0.3cm} 0.4872
        \end{pmatrix}
    \end{align*}
    which represents in the RGB space a light shade of blue-magenta colour: \\$\left( \frac{5}{6}, \frac{5}{6}, 1 \right) \approx (0.8333, 0.8333, 1)$.
\end{example}

\section{Conclusion and future Work}

Building on the work of Burgeth and his co-authors, see \cite{Loewner,BurgethKleefeld,morph_op_mat_im}, we have investigated a characterisation for a previously unexplored application of the log-exp supremum, which was introduced by Maslov \cite{Maslov} as an approximation of the maximum in convex optimisation, for colour morphology.
To do this, we used Burgeth's and Kleefeld's \cite{BurgethKleefeld} bijective mapping to assign each colour from the RGB colour space a colour in the form of a symmetric real $2\times2$ matrix in the HCL bi-cone and used the spectral decomposition of these matrices, the Loewner order and properties of the Rayleigh product to calculate the approximation. 
In particular, we were able to show that the LES is transitive according to equation \eqref{Transitivity} and, in combination with the dilation, makes it associative in colour morphology. The latter is a property that is otherwise only known from binary and grey value morphology and, to our knowledge, the only multidimensional dilation in colour morphology with this property.

Moreover, our findings indicate that while the LES is not uniquely minimal with respect to the set of upper bounds of the input data derived from the Loewner ordering, it does, however, exhibit unique minimality when considered in conjunction with the intersection of all $p$-power upper bounds, as defined in Definition 6. To achieve this, the eigenvalues and the Rayleigh product of the exponentiated matrices were employed to eliminate the problematic cases through the limit as $p\to\infty$.

Finally, we have also addressed one of the biggest downsides of the LES, namely the fact that it does not depend continuously on the input data. However, among other things, this only applied to non-generic configurations of the orthonormal bases of the input data, which allowed us to remove these by means of a relaxation of the LES, the RLES. The input data was changed slightly so that these cases no longer occurred. This change also meant that not only these discontinuities were removed, but all discontinuities were removed without losing the favourable properties of the LES shown above. Since the duality relationship between dilation and erosion is also intact in our method, all the properties of dilation shown here can be transferred to the corresponding properties of erosion. 

In light of these findings, there are several avenues for further investigation of this method in the future. The most notable aspect is the significant increase in brightness of the image. 
One could compare whether the behaviour is similar to that of the paper by Kahra, Sridhar and Breuß, see \cite{KSB}. In addition, given the paper \cite{KahraBreuss} by Kahra and Breuß, it seems possible to find parameters so that these supposed artifacts disappear from \cite{KSB}. However, this would likely be a challenging undertaking. Alternatively, one could also consider the differences in supremum formation with respect to other suprema employed in colour morphology, such as the one-dimensional case or other multidimensional cases, including the trace-supremum \cite{tr_sup} or other suprema \cite{WelkQuantile}, and compare them with each other. 
It should be noted that our method is not an optimal case and could be replaced by selecting a colour space other than the HCL bi-cone, such as Lab, in order to potentially offer additional advantages. The choice of this space was influenced by its graphically clear geometric properties.

\bibliographystyle{splncs04}
\bibliography{MyBibliography}

\end{document}